\documentclass[twoside,11pt]{article}

\usepackage{jmlr2e}
\usepackage{pifont}
\usepackage{amssymb,latexsym,amsmath}
\usepackage{epsfig}
\usepackage{caption}
 \usepackage{float}

 \usepackage{amsmath}
\usepackage{amsfonts}
\usepackage{array}
\usepackage{dsfont}
\usepackage{hyperref}
\usepackage{amssymb}

\usepackage{algorithm}
\usepackage{algorithmic}

\usepackage{makeidx}
\usepackage{graphicx}
\graphicspath{ {./images/} }

\allowdisplaybreaks

 \usepackage{graphicx,fancybox,latexsym,epsfig}
\usepackage{fancyhdr,amsmath,times,amsxtra,amssymb}
\usepackage{color}

\newtheorem{assumption}{Assumption}

\def\Pr{\mathop{\rm Pr}}

\def\P{{\mathcal P}}

\def\sX{{\mathds X}}
\def\sY{{\mathds Y}}
\def\sA{{\mathds A}}

\newcommand{\R}{\mathds{R}}
\newcommand{\Zplus}{\mathbb{Z}_+}
\newcommand{\N}{\mathbb{N}}

\newcommand{\dd}{\mathrm{d}}

\allowdisplaybreaks

\hypersetup{hidelinks}

\firstpageno{1}

\usepackage{lastpage}
\ShortHeadings{Q-Learning for MDPs with General Spaces}{Kara, Saldi and Y\"uksel}

\begin{document}

\title{Q-Learning for MDPs with General Spaces: Convergence and Near
Optimality via Quantization under Weak Continuity\thanks{
This research was supported in part by
the Natural Sciences and Engineering Research Council (NSERC) of Canada.\\
To appear in Journal of Machine Learning Research}
}

\author{\name Ali Devran Kara \email alikara@umich.edu \\
       \addr Department of Mathematics\\
       University of Michigan\\
       Ann Arbor, MI 48109-1043, USA
       \AND
	\name Naci Saldi \email naci.saldi@bilkent.edu.tr\\
	\addr  Department of Mathematics\\
	Bilkent University\\
	Ankara, Turkey
	\AND
       \name  Serdar Y\"uksel \email yuksel@queensu.ca \\
       \addr  Department of Mathematics and Statistics\\
       Queen's University\\
      Kingston, ON, Canada}

\editor{}

\maketitle

\begin{abstract}%
Reinforcement learning algorithms often require finiteness of state and action spaces in Markov decision processes (MDPs) (also called controlled Markov chains) and various efforts have been made in the literature towards the applicability of such algorithms for continuous state and action spaces. In this paper, we show that under very mild regularity conditions (in particular, involving only weak continuity of the transition kernel of an MDP), Q-learning for standard Borel MDPs via quantization of states and actions (called Quantized Q-Learning) converges to a limit, and furthermore this limit satisfies an optimality equation which leads to near optimality with either explicit performance bounds or which are guaranteed to be asymptotically optimal. Our approach builds on (i) viewing quantization as a measurement kernel and thus a quantized MDP as a partially observed Markov decision process (POMDP), (ii) utilizing near optimality and convergence results of Q-learning for POMDPs, and (iii) finally, near-optimality of finite state model approximations for MDPs with weakly continuous kernels which we show to correspond to the fixed point of the constructed POMDP. Thus, our paper presents a very general convergence and approximation result for the applicability of Q-learning for continuous MDPs.
 \end{abstract}


\section{Introduction}

Let $\mathds{X} $ be a Borel set in which the elements of a controlled Markov chain $\{X_t,\, t \in \Zplus\}$ take values.  Here and throughout the paper, $\Zplus$ denotes the set of non-negative integers and $\mathds{N}$ denotes the set of positive integers. Let $\mathds{U}$, the action space, be a compact Borel subset of some Euclidean space, from which the sequence of control action variables $\{U_t,\, t \in \Zplus\}$ take values. 

The $\{U_t, \, t \in \mathbb{Z}_+\}$, are generated via admissible control policies: An {\em admissible policy} $\gamma$ is a sequence of control functions $\{\gamma_t,\, t\in \Zplus\}$ such that $\gamma_t$ is measurable on the $\sigma$-algebra generated by the information variables
\[
I_t=\{X_0,\ldots,X_t,U_0,\ldots,U_{t-1}\}, \quad t \in \mathds{N}, \quad
  \quad I_0=\{X_0\},
\]
where
\begin{equation}
\label{eq_control}
U_t=\gamma_t(I_t),\quad t\in \Zplus,
\end{equation}
are the $\mathds{U}$-valued control
actions.
\noindent We define $\Gamma$ to be the set of all such admissible policies.

The joint distribution of the state and control
processes is then completely determined by (\ref{eq_control}), the initial probability measure of $X_0$, and the following
relationship:
\begin{align}\label{eq_evol}
& \Pr\biggl( X_t\in B \, \bigg|\, (X,U)_{[0,t-1]}=(x,u)_{[0,t-1]} \biggr) = \int_B \mathcal{T}( dx_t|x_{t-1}, u_{t-1}),  B\in \mathcal{B}(\mathds{X}), t\in \mathds{N},
\end{align}
where $\mathcal{T}(\cdot|x,u)$ is a stochastic kernel  (that is, a regular conditional probability measure) from $\mathds{X}\times \mathds{U}$ to $\mathds{X}$, $\mathcal{B}(\mathds{X})$ is the Borel $\sigma$-algebra of $\mathds{X}$, and $(X,U)_{[0,t-1]}$ is the set of state-action pairs up until $t-1$.

The objective of the controller is to minimize the infinite-horizon discounted expected cost
  \begin{align*}
    J_{\beta}(x_0,\gamma)= E_{x_0}^{{\mathcal{T}},\gamma}\left[\sum_{t=0}^{\infty} \beta^tc(X_t,U_t)\right]
  \end{align*}
over the set of admissible policies $\gamma\in\Gamma$, where $0<\beta<1$ is the discount factor, $c:\mathds{X}\times\mathds{U}\to\R$ is the stage-wise continuous and bounded cost function, and $E_{x_0}^{{\mathcal{T}},\gamma}$ denotes the expectation with initial state $x_0$ and transition kernel $\mathcal{T}$ under policy $\gamma$. For any initial state $X_0=x_0$, the optimal value function is defined by
\begin{align*}
  J_{\beta}^*(x_0)&=\inf_{\gamma\in\Gamma} J_{\beta}(x_0,\gamma).
\end{align*}
To calculate the optimal value function and the optimal control policy, various numerical approaches can be adopted, e.g., value iteration, policy iteration, linear programming (\cite{HernandezLermaMCP}) under the assumption that the transition probability $\mathcal{T}$ and the cost function $c$ are known. If the model is unknown, a powerful and popular tool is the Q-learning algorithm by \cite{Watkins}. The Q-learning algorithm provides an iterative approach that is guaranteed to converge under mild assumptions if the model is finite and if the controller has access to the state and cost realizations.

In this paper, we present a very general result on the applicability and near-optimality of Q-learning for setups where the state and action spaces are standard Borel (i.e., Borel subsets of complete, separable and metric spaces).

In what follows, we first provide a review of the related literature and some background.

\subsection{Literature Review}

The Q-learning algorithm  (see \cite{Watkins,TsitsiklisQLearning,Baker,CsabaLittman}) is a stochastic approximation algorithm that does not require the knowledge of the transition kernel, or even the cost (or reward) function for its implementation. In this algorithm, the incurred per-stage cost variable is observed through simulation of a single sample path. When the state and the action spaces are finite, under mild conditions regarding infinitely often hits for all state-action pairs, this algorithm is known to converge to the optimal cost {and arrive at optimal policies}. 

In this paper, our focus is on the case with continuous spaces. While this setup has attracted significant interest in the literature, there remain significant open questions on rigorous approximation or convergence bounds, as we discuss further below.

The approach we present for continuous spaces is via quantization and by viewing quantized MDPs as partially observed Markov decision processes (POMDPs). We will establish convergence and rigorous near optimality results.

\subsubsection{Near optimality of quantized MDPs} 

For MDPs with continuous state spaces, existence for optimal solutions has been well studied. Under either weak continuity of the kernel (in both the state and action), or strong continuity (of the kernel in actions for every state) properties and measurable selection conditions, dynamic programming and Bellman's equations of optimality lead to existence results. The corresponding measurable selection criteria are given by \citet[Theorem 2]{himmelberg1976optimal}, \cite{Schal}, \cite{schal1974selection} and \cite{kuratowski1965general}. We also refer the reader to   \cite{HernandezLermaMCP} for a comprehensive analysis and detailed literature review and \citet[Theorem 2.1]{feinberg2021mdps}. 

However, the above do not directly lead to computationally efficient methods. Accordingly, various approaches have been developed in the literature, with particularly intense recent research activity, to compute approximately optimal policies by reducing the original problem into a simpler one. A partial list of these techniques is as follows: approximate dynamic programming, approximate value or policy iteration, simulation-based techniques, neuro-dynamic programming (or reinforcement learning), state aggregation, etc. \citep[see e.g.][]{DuPr12,Ber75,chow1991optimal,BertsekasTsitsiklisNeuro}. Indeed, for MDPs, numerical methods have been studied under very general models with a comprehensive review available by \cite{SaLiYuSpringer}. Notably, as it is related to our analysis in this paper, \cite{saldi2014near,SaldiLinderYukselTAC14,SaYuLi15e,SaYuLi15c} have shown that under weak continuity conditions for an MDP with standard Borel state and action spaces, finite models obtained by the quantization of the state and action spaces lead to control policies that are asymptotically optimal as the quantization rate increases, where, under further regularity conditions, rates of convergence relating error decay and the number of quantization bins are also obtained. We will make explicit connections throughout our paper.

Despite the above mentioned rigorous results for near optimality under very {\it weak} conditions, a corresponding reinforcement learning result for such quantized MDPs with conclusive results on convergence and near optimality under similarly weak conditions does not yet exist despite many related studies which demand more restrictive conditions. A common approach for reinforcement learning for continuous spaces is through using function approximation for the optimal value function \citep[see][]{CsabaAlgorithms,tsitsiklis1997analysis}. The function approximation is usually done using neural networks, state aggregation, or through a linear approximation with finitely many linearly independent basis functions. For the neural network approximations, the results typically lack convergence proofs. For state aggregation and linear approximation methods, while often convergence is studied, the error analysis regarding the limit of the stochastic iterates is typically not studied in general or an error analysis is not provided at all. Some related work is done by \cite{singh1995reinforcement,melo2008analysis,gaskett1999q,CsabaSmart} and references therein: \cite{melo2008analysis} consider compact models with no error analysis with regard to the limit Q function and the optimal policy. In the context of finite space models, \citet[Chapter 6]{BertsekasTsitsiklisNeuro} and \cite{singh1995reinforcement} consider state-aggregation, with the latter studying a soft version, and establish the convergence of the limit iterates. By considering more general (i.e., continuous) spaces, \cite{CsabaSmart} generalize the above by the use of Q function approximators (interpolators) that are sufficiently regular (defined by non-expansiveness) in their parametric representation and establish both convergence and optimality properties. {A further recent related study along similar lines is \cite{SongWen2019} which imposes Lipschitz regularity conditions on the $Q$ functions}. {Another related direction for continuous models is (model-based) kernel estimation methods \citep[see][]{ormoneit2002kernel,ormoneit2002average,Sinclair2020}}. For kernel estimation methods, it is typically assumed that the transition probabilities admit a highly regular density function, and the density function, and thus the value functions and optimal policies, are learned using approximating kernel functions in a consistent fashion; for this method, independently generated state pairs are used rather than a single sample path. 

Different from the studies above, we will consider MDPs with continuous state and action spaces and with only weakly continuous transition kernels, and establish both convergence and near optimality results. In addition, we will also consider slightly stronger transition kernels, to arrive at stronger convergence results. We note that our approach to be presented can be referred to as state aggregation, although we usually refer to it as the quantization of the state space, and it can also be seen as a linear function approximation where the basis functions are constant over the quantization bins and zero elsewhere. Due to this special approximation structure, we are able to provide a finite MDP model for the limit Q-values, and thus, we can have more insight and intuition on the analysis of the error term (see Remark \ref{compremark} for further discussion).

One particularly related paper that is closely related our setup is by \cite{NNQlearning} where the authors study the finite sample analysis of a quantized Q-learning method via nearest neighbor mapping. \cite{NNQlearning} assume that the transition model admits a Lipschitz continuous density function with respect to the Lebesgue measure. In our work, we study weaker and more general MDP models where we show that only the weak continuity of the transition models are sufficient for asymptotic consistency or Wasserstein type metrics for convergence rates which do not require continuous density assumptions. Furthermore, we use a general mapping for the quantization as opposed to a nearest neighbor map. We also note that \cite{NNQlearning} study finite sample setup with fast mixing conditions on the transition model which in turn implies geometric convergence to the invariant measure of the controlled process, whereas we only focus on the asymptotic time analysis with weaker stability assumptions on the process. 

{One reason for the challenges of reinforcement learning theory for quantized models is that quantized MDPs are no longer true MDPs with respect to the probabilistic flow of a true model (even though as an analytical construction quantized MDPs can be designed to be constructed as actual MDPs towards constructing near optimal policies); this essentially generalizes the intuitive and correct result that when one quantizes a Markov process, the quantized outputs are no longer Markovian}. This question leads us to the next discussion.

\subsubsection{Convergence of Q-learning for POMDPs}  {A stochastic control model where the controller can have access to only a noisy or partial version of the state via measurements is called a Partially Observed Markov Decision Process (POMDP)}. Learning in POMDPs is challenging, mainly due to the non-Markovian behavior of the observation process. A question, which has recently been studied by \cite{kara2021convergence} (see also \cite{kara2020near}) is the following: 
\begin{itemize}
\item[(i)] whether a Q-learning algorithm for such a setup would indeed converge, 
\item[(ii)] if it does, where does it converge to?
\end{itemize} The answer to the first part of the question is positive under mild conditions \citep[see][for the case with unit memory]{singh1994learning} and \cite{CsabaSmart}; and the answer to the second part of the question is; under filter stability conditions, the convergence is to near optimality with an explicit error bound between the performance loss and the memory window size.   \cite{kara2021convergence} provide a detailed analysis and literature review.

To be more concrete, a natural attempt to learn POMDPs would be to ignore the partial observability and pretend that the noisy observations reflect the true state. For example, for infinite horizon discounted cost problems, one can construct Q-iterations as:
\begin{align}\label{QPOMDP}
&Q_{k+1}(y_k,u_k)=(1-\alpha_k(y_k,u_k))Q_k(y_{k},u_k)\nonumber\\
&\phantom{xxxxxxxxxxxxxxx}+\alpha_k(y_k,u_k)\left(C_k(y_k,u_k)+\beta \min_v Q_k(Y_{k+1},v)\right)
\end{align}
where $y_k$ represents the observations and $u_k$ represents the control actions, $0<\beta<1$ is the discount factor, and $\alpha_k$'s are the learning rates. 
We can further improve this algorithm by using not only the most recent observation but a finite window of past observations and control actions. 

However, the joint process of the observation and the control variables is not a controlled Markov process (as only $(X_k,U_k)$ is), and hence the convergence does not follow directly from usual techniques \citep[see][]{jaakkola1994convergence, TsitsiklisQLearning}. Even if the convergence is guaranteed, it is not immediate what the limit Q-values are, and whether they are meaningful at all. In particular, it is not known what MDP model gives rise to the limit Q-values.  \cite{singh1994learning} studied the Q-learning algorithm for POMDPs by ignoring the partial observability and constructing the algorithm using the most recent observation variable as in (\ref{QPOMDP}), and established convergence of this algorithm under mild conditions (notably that the hidden state process is uniquely ergodic under the exploration policy which is random and puts positive measure to all action variables); see also \cite{CsabaSmart}. \cite{kara2021convergence} considered memory sizes of more than zero for the information variables and a continuous state space. It was shown that the Q-iterations constructed using finite history variables converge under mild assumptions on the hidden state process and filter stability, and that the limit fixed point equation corresponds to an optimal solution for an approximate belief-MDP model and established bounds for the performance loss of the policy obtained using the approximate belief-MDP when it is used in the original model.

 {\bf Contributions and Main Results.}
In this paper, we present, in part by unifying and generalizing the aforementioned ingredients, very general results on the convergence and near optimality of Q-learning for quantized MDPs for non-compact state and compact action spaces. We list the main results of the paper as follows:
\begin{itemize}
\item In Section \ref{finite_app_sec}, we study a finite approximation method for continuous MDP models. In particular, 
\begin{itemize}
\item In Theorem \ref{tv_thm} and \ref{tv_thm_robust}, we show that the finite state approximations for models with total variation continuous kernels, are nearly optimal, and the error bound is in terms of the expected accumulated quantization error.

\item In Theorem \ref{wass_thm} and \ref{wass_thm_robust}, we show that the finite state approximations for models with transition kernels continuous under the first order Wasserstein distance, are nearly optimal, and the error bound is in terms of the uniform quantization error.
\item Theorem \ref{weak_thm}, under weak continuity conditions on the kernel, shows that finite state approximations are asymptotically optimal as the number of bins approach infinity.
\end{itemize}

In the following, all the presented results except Theorem 7 hold true for complete, separable and metric spaces (that is, Polish spaces), and not only for Euclidean spaces; for Theorem \ref{weak_thm} we also require the space to be $\sigma$-compact (that is, $\mathds{X} = \cup_{k=1}^{\infty} B_k$ with each $B_k$ compact). However, for clarity in presentation, for most of the results in the following we will consider the state space to be finite dimensional Euclidean, with the generalization to more general metric (Polish) spaces being mostly mechanical, where one needs to replace the norm with the corresponding metric on $\mathds{X}$. 


\item In Section \ref{pomdp_sec}, we construct an approximate Q learning algorithm by viewing the quantized models as an artificial POMDP, and in Theorem \ref{main_thm}, we establish that this approximate (quantized) Q learning algorithm converges to the optimality equation for the finite models constructed in Section \ref{finite_app_sec}. Hence, error bounds are provided for the learned policy in Section \ref{error_analysis}. In particular
\begin{itemize}
\item In Corollary \ref{cor3}, we show that the policies learned via the Q learning algorithm are asymptotically optimal with the increasing quantization rate, when the transition kernel of the model is weakly continuous.
\item In Corollary \ref{cor1}, we establish a convergence rate for the error of the learned policy, when the transition kernel of the model is continuous in total variation, using Theorem \ref{tv_thm} and \ref{tv_thm_robust}.
\item In Corollary \ref{cor2}, we establish a convergence rate for the error of the learned policy, when the transition kernel of the model is continuous under Wasserstein distance (first order), using Theorem \ref{wass_thm} and \ref{wass_thm_robust}.
\end{itemize}
\end{itemize}

The proposed method is explained in detail in Section \ref{model_intro}. The algorithm can be summarized in the following steps:
\begin{itemize}
\item Step 1: Quantize the action space.
\item Step 2: Quantize the state space (since the state space is not compact in general, quantization may be non-uniform).
\item Step 3: Run Q-learning on the finite model via (\ref{q_alg3}).
\item Step 4: Apply the resulting control policy on the true model by extending it to the true state space (e.g. the resulting policy will be a piece-wise constant function on the state space if we use constant extension over the quantization bins).
\end{itemize}


\subsection{Proposed Approach for Continuous Spaces}

In this section, we first review the traditional Q learning algorithm for finite MDP models and we explain the challenges of application of the algorithm to models with continuous state and action spaces. We finally present our proposed approach for the learning problem in continuous models.

\subsubsection{Review of Q-Learning for Fully Observed Finite Models}
We start with the discounted cost optimality equation (DCOE) for finite models given by
\begin{align*}
&J_\beta^*(x)=\min_{u\in\mathds{U}}\left\{c(x,u)+\beta \sum_{y \in \mathds{X}}J^*_\beta(y)\mathcal{T}(y|x,u)\right\}.
\end{align*} 
A real-valued function on the state space $\mathds{X}$ satisfies the DCOE if and only if it is the optimal value function \cite[Theorem 4.2.3]{HernandezLermaMCP}, thus, DCOE is a key tool for the optimality analysis of infinite horizon discounted cost problems. 

Note that the optimal value function is defined for every state. We now introduce the optimal Q-function defined for every state and action pair, which satisfy the following fixed point equation
\begin{eqnarray}\label{Q_alg}
Q^*(x,u)=c(x,u)+\beta\sum_{y \in \mathds{X}}\min_vQ^*(y,v) \mathcal{T}(y|x,u).
\end{eqnarray} 
Furthermore, the optimal Q-function satisfies the following relation for all $x \in \mathds{X}$:
\begin{align*}
\min_{v \in \mathds{U}} Q^*(x,v)=J_\beta^*(x).
\end{align*} 
The deterministic stationary policy that minimizes the above equation for any $x \in \mathds{X}$ is the optimal policy. Hence, if one knows the optimal Q-function, optimal value function and an optimal policy can be calculated. The optimal Q-function can be calculated by applying the contractive operator on the right side of the equation (\ref{Q_alg}) iteratively starting from some initial Q-function. This is the value iteration algorithm for Q-functions and convergence of this algorithm to optimal Q-function follows from the Banach fixed point theorem.

If the transition kernel $\mathcal{T}$ and the stage-wise cost $c$ are not available, one can apply the Q-learning algorithm to obtain optimal Q-function. In this algorithm, the decision maker applies an arbitrary admissible policy $\gamma$ and collects realizations of state, action, and stage-wise cost under this policy:
$$
X_0,U_0,c(X_0,U_0),X_1,U_1,c(X_1,U_1), \ldots. 
$$
Using this collection, it updates its Q-functions as follows: for $t\geq0$, if $(X_t,U_t)=(x,u)$, then
    \begin{eqnarray}\label{FiP}
    Q_{t+1}(x,u)  =   Q_t(x,u)   +  \alpha_t(x,u) \left( c(x,u)+\beta \min_{v \in \mathds{U}} Q_t(X_{t+1},v) -Q_t(x,u) \right)
    \end{eqnarray}
where the initial condition $Q_0$ is given, $\alpha_t(x,u)$ is the step-size for $(x,u)$ at time $t$, $U_t$ is chosen according to exploration policy $\gamma$, and the random state $X_{t+1}$ evolves according to $\mathcal{T}(X_{t+1} \in \cdot \ | X_t=x,U_t=u)$ starting at $X_0=x$. Under the following conditions, the iterations given in (\ref{FiP}) will converge to the optimal Q-function $Q^*$ almost surely. 

\smallskip 

\begin{assumption}\label{as:alpha2}
For all $(x,u)$ and for all $t \geq 0$, we have
\begin{itemize}
\item[a)] $\alpha_t(x,u)\in[0,1].$
\item[b)] $\alpha_t(x,u)=0$ unless $(x,u)=(X_t,U_t).$
\item[c)] $\alpha_t(x,u)$ is a (deterministic) function of $(X_{0},U_{0}),\dots,(X_t,U_t)$. 
\item[d)] $\sum_{t \geq 0} \alpha_t(x,u) = \infty$, almost surely.
\item[e)] $\sum_{t \geq 0} \alpha_t^2(x,u) \leq C$, almost surely, for some constant $C<\infty$.
\end{itemize}
\end{assumption}

\smallskip


Hence, Q-learning iterations can be used to calculate the optimal value function and an optimal policy if one has access to state and stage-wise cost realizations. However, this approach is tailored for finite models, and in particular the assumption that every $(x,u)$ pair is visited infinitely often is not feasible for continuous state and action spaces.

\subsubsection{Challenges and the Proposed Approach for Continuous Spaces}\label{model_intro}

As we noted in the previous section, for continuous spaces, one cannot visit every sate and action pair infinitely often. Hence, traditional Q-learning algorithm given in (\ref{FiP}) is not directly applicable. To overcome this obstacle, we will reduce the original problem to a finite one by quantizing the state space and modify the Q-iteration algorithm accordingly. For the moment, we assume that the action space $\mathds{U}$ is finite; this will be addressed later (in particular, we will show that under mild conditions to be given which only involve weak continuity, $\mathds{U}$ can be replaced with a finite subset for any approximation error tolerance).  

Let $\mathds{Y} \subset \mathds{X}$ be a finite set, which approximates the original state space $\mathds{X}$ of the model. Define a mapping $q:\mathds{X}\to\mathds{Y}$ such that for any $x\in\mathds{X}$, $q(x)=y$ for some $y\in\mathds{Y}$. For a continuous state space $\mathds{X}$, $q$ can be seen as the discretization mapping. For example, one can choose a collection of disjoint Borel measurable sets $\{B_i\}_{i=1}^M$ such that $\bigcup_i B_i=\mathds{X}$ and $B_i\bigcap B_j =\emptyset$ for any $i\neq j$. Furthermore, one can choose a representative state, $y_i\in B_i$, from each disjoint set. For this setting, we have 
$\mathds{Y}:=\{y_1,\dots,y_M\}$ and
\begin{align*}
q(x)=y_i \quad \text{ if } x\in B_i.
\end{align*}
We note that the set $\mathds{Y}=\{y_1,\dots,y_M\}$ is only a set of representative states for the collection of disjoint sets $\{B_i\}_{i=1}^M$ and the actual values of the representative states do not affect the error analysis and the overall algorithm performance. Instead, the collection of sets $\{B_i\}_{i=1}^M$ is the key element of the performance. The  sets $\{B_i\}_{i=1}^M$ can be chosen depending on the application, e.g., to minimize the loss function $L(x)$ (see (\ref{loss_func})). Quantization theory deals with the optimal design of such maps under a cardinality constraint  \citep[see][]{GrayNeuhoff}.

We also note that even for a finite state space $\mathds{X}$, in order to reduce size of the state space, one can choose a smaller set $\mathds{Y}$, by collecting multiple states in one group. 

Using the map $q$, we construct the following Q-learning algorithm. Again, the decision maker applies an arbitrary admissible {exploration} policy $\gamma$ and collects realizations of state, action, and stage-wise cost under this policy:
$$
X_0,U_0,c(X_0,U_0),X_1,U_1,c(X_1,U_1), \ldots. 
$$  
Using this collection, it updates its Q-functions defined only for state-action pairs in $\mathds{Y}\times\mathds{U}$ as follows: for $t\geq0$, if $(X_t,U_t)=(x,u)\in\mathds{X}\times\mathds{U}$, then
\begin{align}\label{q_alg2}
&Q_{t+1}(q(x),u)=(1-\alpha_t(q(x),u)) \, Q_t(q(x),u) \nonumber \\
&\phantom{xxxxxxxxxxxxxxxx}+\alpha_t(q(x),u)\left(c(x,u)+\beta \min_{v \in \mathds{U}} Q_t(q(X_{t+1}),v)\right), 
\end{align}
that is, for any true value of the state, we use its representative state from the finite set $\mathds{Y}$ when updating the Q-function.

We have now reduced the iterations to a finite set $\mathds{Y}\times\mathds{U}$, and therefore, it is feasible to visit  every pair $(y,u)$ infinitely often. However, one cannot directly argue that the iterations in (\ref{q_alg2}) will converge to a Q-function satisfying some fixed point equation. Even if the convergence is guaranteed, one needs to give a meaning to the limit fixed point equation, i.e., we need to construct the approximate model whose optimal Q-function satisfies the limit fixed point equation. Two main challenges for the convergence are the following:
\begin{itemize}
\item For the convergence of the traditional Q-iterations defined in (\ref{FiP}), it is a crucial assumption that the state process $X_t$ is controlled Markov chain. However, the process $q(X_t)$ is not a controlled Markov chain; that is, for all $t\geq0$, 
\begin{eqnarray*}
& \Pr\biggl( q(X_t) \, \bigg|\, (q(X),U)_{[0,t-1]} \biggr) \neq \Pr\biggl( q(X_t) \, \bigg|\, q(X_{t-1}),U_{t-1} \biggr).
\end{eqnarray*}
The reason is that $q(X_{t-1})$ gives only partial information about the true state $X_{t-1}$. 
\item In the algorithm, for any iteration step $t\geq0$, the realized stage-wise cost $c(X_t,U_t)$ is not a function of the quantized state $q(X_t)$. 
\end{itemize} 

\noindent To overcome these challenges, we will view the approximate finite model as a partially observed MDP (POMDP) where the map $q$ induces a quantizer channel from $\mathds{X}$ to $\mathds{Y}$, so that, the controller does not have full access to the true value of the state $x$ but it observes a quantized version of the true state value. We will then use recent results for convergence of Q-learning algorithms for POMDPs by \cite{kara2021convergence}. Finally, we will prove that the limiting Q-function is the optimal Q-function of the finite approximate model introduced by \cite{SaLiYuSpringer,SaYuLi15c}, and provide error bounds using the finite approximate models.

The rest of the paper is organized as follows. In Section \ref{finite_app_sec}, we introduce finite approximate models for MDPs with continuous spaces, and we provide various error bounds for the approximate models under different set of assumptions on the system components. In Section \ref{pomdp_sec}, we introduce an approximate Q-learning algorithm for POMDPs, and we establish the connection between POMDPs and quantized approximate models. In particular, we show that the Q-iterations defined in (\ref{q_alg2}) converges to the optimal Q-function of the finite model introduced in Section \ref{finite_app_sec}. Finally, in Section \ref{num_stud}, we present two case studies.
\section{Near Optimality of Finite Model Approximations}\label{finite_app_sec}

In this section, we introduce finite MDP models, building on \cite{SaLiYuSpringer,SaYuLi15c} with more general conditions (e.g. to allow for non-uniform quantization), for MDPs with continuous state and action spaces. 

\subsection{Convergence Notions for Probability Measures and Regularity Properties of Transition Kernels}
For the analysis of the technical results, we will use different notions of convergence for sequences of probability measures.

Two important notions of convergence for sequences of probability measures are weak convergence and convergence under total variation. For some $N\in\N$, a sequence $\{\mu_n,n\in\N\}$ in $\mathcal{P}(\mathds{X})$ is said to converge to $\mu\in\mathcal{P}(\mathds{X})$ \emph{weakly} if $\int_{\mathds{X}}c(x)\mu_n(dx) \to \int_{\mathds{X}}c(x)\mu(dx)$ for every continuous and bounded $c:\mathds{X} \to \R$.

  For probability measures $\mu,\nu \in \mathcal{P}(\mathds{X})$, the \emph{total variation} metric is given by
  \begin{align*}
    \|\mu-\nu\|_{TV}&=2\sup_{B\in\mathcal{B}(\mathds{X})}|\mu(B)-\nu(B)|=\sup_{f:\|f\|_\infty \leq 1}\left|\int f(x)\mu(\dd x)-\int f(x)\nu(\dd x)\right|,
  \end{align*}
  \noindent where the supremum is taken over all measurable real $f$ such that $\|f\|_\infty=\sup_{x\in\mathds{X}}|f(x)|\leq 1$. A sequence $\mu_n$ is said to converge in total variation to $\mu \in \mathcal{P}(\mathds{X})$ if $\|\mu_n-\mu\|_{TV}\to 0$.

 {Finally, for probability measures $\mu,\nu \in \mathcal{P}(\mathds{X})$ with finite first order moments (that is, $\int \|x\| \, d\nu$ and $\int \|x\| \, d\mu $ are finite)}, the \emph{first order Wasserstein} distance is defined as 
\begin{align*}
W_1(\mu,\nu)=\inf_{\Gamma(\mu,\nu)}E[|X-Y|]=\sup_{f: Lip(f)\leq 1}|\int f(x)\mu(dx)-\int f(x)\nu(dx)|
\end{align*}
where $\Gamma(\mu,\nu)$ denotes the all possible couplings of $X$ and $Y$ with marginals $X\sim\mu$ and $Y\sim\nu$, {and 
\begin{align}
Lip(f) := \sup_{e \neq e'} \frac{f(e) - f(e')}{\|e - e'\|},\nonumber
\end{align}}
 and the second {equality follows from the dual formulation of the Wasserstein distance \cite[Remark 6.5]{Vil09}}. Note that the weak convergence and the Wasserstein convergence are equivalent if the underlying space is compact.

We can now define the following regularity properties for the transition kernels:
\begin{itemize}
\item $\mathcal{T}(\cdot|x,u)$ is said to be weakly continuous in $(x,u)$, if $\mathcal{T}(\cdot|x_n,u_n)\to \mathcal{T}(\cdot|x,u)$ weakly for any $(x_n,u_n)\to (x,u)$.
\item $\mathcal{T}(\cdot|x,u)$ is said to be continuous under total variation in  $(x,u)$, if $\|\mathcal{T}(\cdot|x_n,u_n)- \mathcal{T}(\cdot|x,u)\|_{TV}\to 0$  for any $(x_n,u_n)\to (x,u)$.
\item $\mathcal{T}(\cdot|x,u)$ is said to be continuous under the first order Wasserstein distance in $(x,u)$, if \[W_1(\mathcal{T}(\cdot|x_n,u_n), \mathcal{T}(\cdot|x,u))\to 0\]  for any $(x_n,u_n)\to (x,u)$. {To ensure continuity of $\mathcal{T}$ with respect to the first order Wasserstein distance, in addition to weak continuity, we may assume that there exists a function $g:[0,\infty) \rightarrow [0,\infty)$ such that as $t \to \infty$, $\frac{g(t)}{t} \uparrow \infty$, and 
$$\sup_{(x,u) \in K \times \mathds{U}} \int g(\|y\|) \, \mathcal{T}(dy|x,u) < \infty$$ for any compact $K \subset \mathds{X}$. Note that the latter condition implies uniform integrability of the collection of random variables with probability measures ${\cal T}(dx_1|X_0=x_n,U_0=u_n)$ as $(x_n, u_n) \to (x,u)$, which coupled with weak convergence can be shown to imply convergence under the Wasserstein distance.}
\end{itemize}
\begin{example}\label{examples}
Some example models satisfying these regularity properties are as follows:
\begin{itemize}
\item[(i)] For a model with the dynamics $x_{t+1}=f(x_t,u_t,w_t)$, the induced transition kernel $\mathcal{T}(\cdot|x,u)$ is weakly continuous in $(x,u)$ if $f(x,u,w)$ is a continuous function of $(x,u)$, since for any continuous and bounded function $g$
\begin{align*}
&\int g(x_1)\mathcal{T}(dx_1|x_n,u_n)=\int g(f(x_n,u_n,w))\mu(dw)\\
&\to\int g(f(x,u,w))\mu(dw)=\int g(x_1)\mathcal{T}(dx_1|x,u)
\end{align*}
where $\mu$ denotes the probability measure of the noise process.
{If we also have that $\mathds{X}$ is compact, the transition kernel $\mathcal{T}(\cdot|x,u)$ is also continuous under the first order Wasserstein distance}.

\item[(ii)] For a model with the dynamics $x_{t+1}=f(x_t,u_t)+w_t$, the induced transition kernel $\mathcal{T}(\cdot|x,u)$ is continuous under total variation in $(x,u)$ if $f(x,u)$ is a continuous function of $(x,u)$, and $w_t$ admits a continuous density function. 
\item[(iii)] In general, if the transition kernel admits a continuous density function $f$ so that $\mathcal{T}(dx_1|x,u)=f(x_1,x,u)dx_1$, then $\mathcal{T}(dx_1|x,u)$ is continuous in total variation. This follows from an application of  Scheff\'e's Lemma \cite[Theorem 16.12]{Bil95}. In particular, we can write that
\begin{align*}
\|\mathcal{T}(\cdot|x_n,u_n)-\mathcal{T}(\cdot|x,u)\|_{TV}=\int_{\mathds{X}}|f(x_1,x_n,u_n)-f(x_1,x,u)|dx_1\to 0.
\end{align*}
\item[(iv)] For a model with the dynamics $x_{t+1}=f(x_t,u_t,w_t)$, if $f$ is Lipschitz continuous in $(x,u)$ pair such that, there exists some $\alpha<\infty$ with
\begin{align*}
|f(x_n,u_n,w)-f(x,u,w)|\leq \alpha\left(|x_n-x|+|u_n-u|\right),
\end{align*}
we can then bound the first order Wasserstein distance between the corresponding kernels with $\alpha$:
\begin{align*}
&W_1\left(\mathcal{T}(\cdot|x_n,u_n),\mathcal{T}(\cdot|x,u)\right)=\sup_{Lip(g)\leq 1}\left|\int g(x_1)\mathcal{T}(dx_1|x_n,u_n)-\int g(x_1)\mathcal{T}(dx_1|x,u) \right|\\
&=\sup_{Lip(g)\leq 1}\left|\int g(f(x_n,u_n,w))\mu(dw)-\int g(f(x,u,w))\mu(dw)\right|\\
&\leq \int \left|f(x_n,u_n,w)-f(x,u,w)\right| \mu(dw)\leq \alpha\left(|x_n-x|+|u_n-u|\right).
\end{align*}
\end{itemize}
\end{example}

\subsection{Finite Action Approximate MDP: Quantization of the Action Space and Near Optimality of Finite Action Models}\label{finiteActionMDP}

Let $d_{\mathds{U}}$ denote the metric on $\mathds{U}$. Since the action space $\mathds{U}$ is compact, one can find a sequence of finite sets $\Lambda_n = \{u_{n,1},\ldots,u_{n,k_n}\} \subset \mathds{U}$ such that for all $n$,
$$
\min_{i\in\{1,\ldots,k_n\}} d_{\mathds{U}}(u,u_{n,i}) < 1/n \text{ for all } u \in \mathds{U}.
$$
In other words, $\Lambda_n$ is a $1/n$-net in $\mathds{U}$. For any $f: \mathds{X} \rightarrow \mathds{U}$, let us define the mapping 
\begin{align}
\Upsilon_n(f)(x) := \mathop{\rm arg\, min}_{u\in\Lambda_n} d_{\mathds{U}}(f(x),u) \label{strong:neq6},
\end{align}
where ties are broken so that $\Upsilon_n(f)(x)$ is measurable. The following assumption is imposed on the transition kernel so that the true MDP model can be approximated well by the MDPs with finite action spaces. 

\smallskip

\begin{assumption}
\label{weak:as1}
\begin{itemize}
\item[(i)]The stochastic kernel ${\cal T}(\,\cdot\,|x,u)$ is weakly continuous in $(x,u) \in {\mathds X} \times \mathds{U}$.
\item[(ii)] $c: {\mathds X} \times \mathds{U} \to \mathds{R}_+$ is continuous and bounded.
\end{itemize}
\end{assumption}

\smallskip

For any real-valued continuous and bounded function $v$ on ${\mathds X}$, let $\mathbb{T}$ be given by 
\begin{align}
(\mathbb{T}v) (x) := \min_{u \in \mathbb{U}} \biggl[ c(x,u) + \beta \int_{{\mathds X}} v(y) \, {\cal T}(dy|x,u) \biggr] \label{weak:eq9}.
\end{align}
Here, $\mathbb{T}$ is the \emph{Bellman optimality operator} for the MDP. Analogously, let us define the Bellman optimality operator $\mathbb{T}_n$ of MDP$_n$, which is defined as the  MDP with finite action space $\Lambda_n$, as
\begin{align}
\mathbb{T}_n v(x) := \min_{u\in\Lambda_n} \biggl[ c(x,u) + \beta \int_{{\mathds X}} v(y) \, {\cal T}(dy|x,u) \biggr] \label{weak:eq13}.
\end{align}
Both $\mathbb{T}$ and $\mathbb{T}_n$ are contraction operators on bounded and continuous functions. Furthermore, value functions of MDP and MDP$_n$ are fixed points of these operators; that is, $\mathbb{T}J^*_{\beta}=J^*_{\beta}$ and $\mathbb{T}_nJ_{\beta,n}^*=J_{\beta,n}^*$. Let us define $v^0 = v_n^0 \equiv 0$, and $v^{t+1} = \mathbb{T} v^t$ and $v_n^{t+1} = \mathbb{T}_n v_n^t$ for $t\geq1$; that is, $\{v^t\}_{t\geq1}$ and $\{v_n^t\}_{t\geq1}$ are successive approximations to the discounted value functions of the MDP and MDP$_n$, respectively (via value iteration). The following result is established using the compactness of the action space and weak continuity of the transition kernel. 

\begin{lemma}{\cite[Lemma 3.19]{SaLiYuSpringer}} \label{weak:lemma1}
Under Assumption \ref{weak:as1}, for any compact $K\subset \mathds{X}$ and for any $t\geq1$, we have
\begin{align}
\lim_{n\rightarrow\infty} \sup_{x\in K} |v_n^t(x) - v^t(x)| = 0. \label{weak:neq1}
\end{align}
\end{lemma}

\noindent The following theorem states that the optimal value function of MDP$_n$ converges to the optimal value function of the original MDP. It can be proved by using Lemma~\ref{weak:lemma1} and taking into account that $\{v^t\}_{t\geq1}$ and $\{v_n^t\}_{t\geq1}$ are successive approximations to the value functions $J^*_{\beta}$ and $J_{\beta,n}^*$, respectively.

\begin{theorem}{\cite[Theorem 3.16]{SaLiYuSpringer}} \label{weak:thm2} 
Under Assumption \ref{weak:as1}, for any compact $K\subset \mathds{X}$, we have
\begin{align}
\lim_{n\rightarrow\infty} \sup_{x\in K} |J_{\beta,n}^*(x) - J_\beta^*(x)| &= 0. \label{weak:eq16}
\end{align}
\end{theorem}

Since any MDP with weakly continuous transition probability can be approximated by MDPs with finite action spaces by Theorem~\ref{weak:thm2}, in the sequel, we assume that $\mathds{U}$ is finite.

\subsection{Finite State Approximate MDP: Quantization of the State Space}\label{finitestate}

We now establish near optimality under finite state approximations. We start by choosing a collection of disjoint sets $\{B_i\}_{i=1}^M$ such that $\bigcup_i B_i=\mathds{X}$, and $B_i\bigcap B_j =\emptyset$ for any $i\neq j$. Furthermore, we choose a representative state, $y_i\in B_i$, for each disjoint set. For this setting, we denote the new finite state space by
$\mathds{Y}:=\{y_1,\dots,y_M\}$, and the mapping from the original state space $\mathds{X}$ to the finite set $\mathds{Y}$ is done via
\begin{align}\label{quant_map}
q(x)=y_i \quad \text{ if } x\in B_i.
\end{align}
Furthermore, we choose a weighting measure $\pi^*\in\P(\mathds{X})$ on $\mathds{X}$ such that $\pi^*(B_i)>0$ for all $B_i$. We now define normalized measures using the weight measure on each separate quantization bin $B_i$ as follows:  
\begin{align}\label{norm_inv}
\hat{\pi}_{y_i}^*(A):=\frac{\pi^*(A)}{\pi^*(B_i)}, \quad \forall A\subset B_i, \quad \forall i\in \{1,\dots,M\},
\end{align}
that is, $\hat{\pi}_{y_i}^*$ is the normalized weight measure on the set $B_i$, where $y_i$ belongs to. 

We now define the stage-wise cost and transition kernel for the MDP with this finite state space $\mathds{Y}$ using the normalized weight measures. Indeed, for any $y_i,y_j\in \mathds{Y}$ and $u \in \mathds{U}$, the stage-wise cost and the transition kernel for the finite-state model are defined as 
\begin{align}\label{finite_cost}
C^*(y_i,u) &= \int_{B_i}c(x,u) \, \hat{\pi}_{y_i}^*(dx),\nonumber\\
P^*(y_j|y_i,u) &= \int_{B_i}\mathcal{T}(B_j|x,u)\, \hat{\pi}_{y_i}^*(dx).
\end{align}
Having defined the finite state space $\mathds{Y}$, the cost function $C^*$ and the transition kernel $P^*$, we can now introduce the optimal value function for this finite model. We denote the optimal value function which is defined on $\mathds{Y}$ by $\hat{J}_\beta:\mathds{Y}\to \mathds{R}$. Note that $\hat{J}_\beta$ satisfies the following DCOE for any $y\in\mathds{Y}$
\begin{align*}
\hat{J}_\beta(y)=\inf_{u\in\mathds{U}}\left\{C^*(y,u)+\beta\sum_{z\in\mathds{Y}}\hat{J}_\beta(z)P^*(z|y,u)\right\}.
\end{align*}
Note that we can easily extend this function over the original state space $\mathds{X}$ by making it constant over the quantization bins. In other words, if $y\in B_i$, then for any $x\in B_i$, we write
\begin{align*}
\hat{J}_\beta(x):=\hat{J}_\beta(y).
\end{align*}
We further define an average loss function $L:\mathds{X}\to\mathds{R}$ as a result of the quantization. For some $x\in\mathds{X}$, where $x$ belongs to a quantization bin $B_i$ whose representative state is $y_i$ (i.e. $q(x)=y_i$), the average loss function $L(x)$ is defined as
\begin{align}\label{loss_func}
L(x):=\int_{B_i}\|x-x'\| \, \hat{\pi}_{y_i}^*(dx') \qquad \forall x \in B_i, i=1,\cdots,M.
\end{align}
That is, $L(x)$ can be seen as the distance of the state $x$ to the mean of the bin $B_i$ under the measure $\hat{\pi}^*_{y_i}$.

In the following, we present error analyses in finite state approximations defined in this section.

\subsubsection{Finite State Approximations with Kernels Continuous in Total Variation under Expected Quantization Error Bounds}
In this section, we focus on an MDP model whose {transition kernel is Lipschitz continuous in $x$ (uniform in $u$) under the total variation norm. This condition is somewhat different than the continuity of the transition kernel under the total variation distance. Indeed, if we have a model as in Example~\ref{examples}-(ii), then we have the required Lipschitz continuity of the transition kernel when $f$ is Lipschitz continuous in $x$ that is uniform in $u$ and the density of the noise $w$ is Lipschitz continuous.} The following assumptions are imposed on the system. 

\smallskip

\begin{assumption}\label{tv_assmpt}
\begin{itemize}
\item[(a)] There exists a constant $\alpha_c>0$ such that $|c(x,u)-c(x',u)|\leq \alpha_c \|x-x'\|$ for all $x,x'\in\mathds{X}$ and for all $u\in\mathds{U}$.
\item[(b)] There exists a constant $\alpha_T>0$ such that $\|\mathcal{T}(\cdot|x,u)-\mathcal{T}(\cdot|x',u)\|_{TV}\leq \alpha_T\|x-x'\|$ for all $x,x'\in\mathds{X}$ and for all $u\in\mathds{U}$.
\end{itemize}
\end{assumption}

{Note that the cost function $c$ is Lipschitz continuous in $x$ (uniform in $u$), if the partial derivative of $c$ with respect to $x$ is uniformly bounded.}

\smallskip

The first result gives an error bound for the approximate value function.
\begin{theorem}\label{tv_thm}
Under Assumption \ref{tv_assmpt}, provided that $c$ is bounded, we have for any initial state $x_0\in\mathds{X}$
\begin{align*}
\left|\hat{J}_\beta(x_0)-J^*_\beta(x_0)\right|\leq \left(\alpha_c+\frac{\beta\alpha_T\|c\|_\infty}{1-\beta}\right) \sum_{t=0}^\infty \beta^t \sup_{\gamma\in\Gamma} E^{\cal{T},\gamma}_{x_0}\left[L(X_t)\right],
\end{align*}
where $L$ is defined in (\ref{loss_func}).
\end{theorem}
\begin{proof}
The proof can be found in Appendix \ref{tv_thm_proof}.
\end{proof}

The following result provides an error bound for the approximate policy of the finite-state model when it is applied to the original model.

\begin{theorem}\label{tv_thm_robust}
Under Assumption \ref{tv_assmpt}, we have for any initial state $x_0\in\mathds{X}$
\begin{align*}
\left|J_\beta(x_0,\hat{\gamma})-J^*_\beta(x_0)\right|\leq 2\left(\alpha_c+\frac{\beta\alpha_T\|c\|_\infty}{1-\beta}\right) \sum_{t=0}^\infty \beta^t \sup_{\gamma\in\Gamma} E^{\cal{T},\gamma}_{x_0}\left[L(X_t)\right]
\end{align*}
where $L$ is defined in (\ref{loss_func}) and $\hat{\gamma}$ denotes the optimal policy of the finite-state approximate model given by (\ref{finite_cost}) extended to the state space $\mathds{X}$ via the quantization function $q$.
\end{theorem}
\begin{proof}
The proof can be found in Appendix \ref{tv_thm_robust_proof}.
\end{proof}


\subsubsection{Finite State Approximation with Kernels Continuous in Wasserstein Distance under Uniform Quantization Error Bounds}

In this section, we focus on the models with transition kernels that are Lipschitz continuous in $x$ (uniform in $u$) under the first order Wasserstein distance. {If we have a model as in Example~\ref{examples}-(i), then we have the required Lipschitz continuity of the transition kernel when $f$ is Lipschitz continuous in $x$ that is uniform in $u$.} Here, instead of providing an average loss bound using (\ref{loss_func}) as in Theorem \ref{tv_thm}, we will provide a uniform loss bound result, and we also assume the state space to be compact. We first define
\begin{align}\label{unif_loss}
\bar{L}:=\max_{i=1,\dots,M}\sup_{x,x'\in B_i}\|x-x'\|.
\end{align}
Here, $\bar{L}$ is the largest diameter among the quantization bins. The following assumptions are imposed on the components of the model.

\smallskip

\begin{assumption}\label{wass_assmpt}
\begin{itemize}
\item[(a)] $\mathds{X}$ is compact. 
\item[(b)] There exists a constant $\alpha_c>0$ such that $|c(x,u)-c(x',u)|\leq \alpha_c \|x-x'\|$ for all $x,x'\in\mathds{X}$ and for all $u\in\mathds{U}$.
\item[(c)] There exists a constant $\alpha_T>0$ such that $W_1(\mathcal{T}(\cdot|x,u),\mathcal{T}(\cdot|x',u))\leq \alpha_T\|x-x'\|$ for all $x,x'\in\mathds{X}$ and for all $u\in\mathds{U}$.
\end{itemize}
\end{assumption}

\smallskip

Note that Assumption~\ref{wass_assmpt}-(a) ensures that the quantity $\bar{L}$ is finite for each $M$ and converges to $0$ as $M \rightarrow \infty$. Without this assumption, $\bar{L} = \infty$ for any $M$. We now present the main results of this section. The first theorem states that the optimal value function of the finite-state model converges to the optimal value function of the original model as $\bar{L} \rightarrow 0$. 

\begin{theorem}\label{wass_thm}
Under Assumption \ref{wass_assmpt}, we have 
\begin{align*}
\sup_{x_0\in\mathds{X}}\left|\hat{J}_\beta(x_0)-J^*_\beta(x_0)\right|\leq \frac{\alpha_c}{(1-\beta\alpha_T)(1-\beta)}\bar{L}
\end{align*}
where $\bar{L}$ is defined in (\ref{unif_loss}).
\end{theorem}

\begin{proof}
The proof can be found in Appendix \ref{wass_thm_proof}.
\end{proof}

The following result is very similar to \cite[Theorem 4.38]{SaLiYuSpringer} with a slightly better bound. It can be established in a straightforward way using Theorem~\ref{wass_thm}. 

\begin{theorem}\label{wass_thm_robust}
Under Assumption \ref{wass_assmpt}, we have 
\begin{align*}
\sup_{x_0\in\mathds{X}}\left|J_\beta(x_0,\hat{\gamma})-J^*_\beta(x_0)\right|\leq  \frac{2\alpha_c}{(1-\beta)^2(1-\beta\alpha_T)}\bar{L}.
\end{align*}
where $\bar{L}$ is defined in (\ref{unif_loss}) and $\hat{\gamma}$ denotes the optimal policy of the finite-state approximate model extended to the state space $\mathds{X}$ via the quantization function $q$.
\end{theorem}
\begin{proof}
The proof can be found in Appendix \ref{wass_thm_robust_proof}.
\end{proof}

\subsubsection{Finite State Approximation with Weakly Continuous  Kernels and Asymptotic Convergence}

In this section, we assume that $\mathds{X}$ is $\sigma$-compact. That is, we can write $\mathds{X} = \cup_{k=1}^{\infty} B_k$ where each $B_k$ is compact. A finite dimensional Euclidean space is an example of such a space. Additionally, in this section, we focus on the models with transition kernels that are continuous only under the weak convergence topology. Here, instead of providing a rate of convergence, we will provide an asymptotic result. Let the quantizer be such that the $M^{th}$ bin be the over-flow bin; that is, the first $ M-1$ bins be the quantization of a compact set and the complement be assigned to $B_M$. To this end, let us define
\begin{align}\label{unif_loss-2}
L^-:=\max_{i=1,\dots,M-1}\sup_{x,x'\in B_i}\|x-x'\|.
\end{align}
Note that since $\mathds{X}$ is $\sigma$-compact, for each $M$, one can find a partition $\{B_i\}_{i=1}^M$ of the state space $\mathds{X}$ such that $L^- \rightarrow 0$ and $\bigcup_{i=1}^{M-1} B_i \nearrow \mathds{X}$ as $M \rightarrow \infty$. Note that $B_M = \mathbb{X} \setminus (\cup_{i=1}^{M-1} B_i)$. In the  following result, we assume that such a sequence of partitions is used to obtain the finite-state approximate models.  

\begin{theorem}{\cite[Theorem 4.27]{SaLiYuSpringer}}\label{weak_thm}
Under Assumption~\ref{weak:as1}, we have for any compact $K \subset \mathds{X}$
\begin{align*}
\sup_{x_0 \in K} \left|\hat{J}_\beta(x_0)-J^*_\beta(x_0)\right| &\to 0
\end{align*}
and
\begin{align*}
\sup_{x_0 \in K}\left|J_\beta(x_0,\hat{\gamma})-J^*_\beta(x_0)\right|&\to 0
\end{align*}
as $L^- \rightarrow 0$, where $\hat{\gamma}$ denotes the optimal policy of the finite-state approximate model extended to the state space $\mathds{X}$ via the quantization function $q$.
\end{theorem}

We note that the result by \citet[Theorem 4.27]{SaLiYuSpringer} is more general and applicable to unbounded cost functions as well. Under the bounded cost in Assumption \ref{weak:as1}, \citet[Theorem 4.27]{SaLiYuSpringer} implies Theorem \ref{weak_thm} above.

Weak continuity, or continuity under Wasserstein distance of the kernels are significantly less restrictive compared to other regularity conditions used in the literature for proving convergence and consistency results, where it is usually assumed that the transition models admit continuous probability density functions. 

On the other hand various MDP models used in the literature fail to satisfy total variation continuity or continuous density assumptions, and can only be shown to weakly continuous. 
\begin{itemize}
\item The belief MDP model, reduction of POMDPs to fully observed counterparts, can be shown to have weakly continuous transition models, see the papers by  \citet{KSYWeakFellerSysCont}  and \citet{FeKaZg14}. However, belief MDPs fail to satisfy stronger continuity assumptions in general.

\item Degenerate controlled diffusion models with piece-wise (in time) constant control policies can be viewed as MDPs, and they can be shown to have weakly continuous  transition models \citep{bayraktarKara}, whereas they fail to admit transition models continuous under total variation.

\item Similarly, mean field control problems, where  several exchangeable agents aim to minimize a common cost, can be posed as probability measure valued MDPs. The measure valued MDP can be shown to weakly continuous \citep{bayraktarKaraMean}
\end{itemize}

\subsubsection{Comparison and discussion}
Theorem \ref{tv_thm_robust} provides a bound that depends on the expectation of the loss function $L(x)$ defined in (\ref{loss_func}). This is, in particular, applicable when uniform quantization cannot be applied for a finite model approximation, which is the case when $\mathds{X}$ is not compact. The error bound in Theorem \ref{wass_thm}, however, uses a uniform bound defined in (\ref{unif_loss}). Although Theorem \ref{tv_thm_robust} requires a stronger assumption, the total variation continuity of the kernel instead of weak convergence metrics, the expected quantization error analysis (as opposed to the uniform error) provides the designer with more freedom to optimize the performance of the algorithm. Note that the uniform bound defined in (\ref{unif_loss}) does not depend on the weight measures $\{\hat{\pi}^*_{y_i}\}$ and is always fixed for every time step. Whereas, $L(x)$ defined in (\ref{loss_func}), uses the weight measures $\{\hat{\pi}^*_{y_i}\}$, for the distance of $x$ to the average of the quantization bin. Furthermore, $L(x)$ is not fixed for every time step, instead, at every time step, one needs the expected distance of $X_t$ to the mean of its quantization bin under the corresponding weight measure. Hence, using the expected loss function $L(x)$, one can adjust the quantization or the selection of the sets $\{B_i\}$'s accordingly, e.g., by performing finer quantization for the more often visited states. 

However, for the uniform bound in Theorem \ref{wass_thm}, to minimize $\bar{L}$ defined in (\ref{unif_loss}), one can only focus on the selection of the sets $\{B_i\}$'s by minimizing the diameter of the maximum possible set in the collection, which indeed provides a cruder upper bound, when $\mathds{X}$ is compact.

Theorem \ref{weak_thm}, on the other hand, requires only weak continuity and the state space does not need to be compact (and, accordingly, the quantizers are also not uniform). Thus, the model is applicable to many practical setups. On the other hand, Theorem \ref{weak_thm} does not provide a rate of convergence.

\section{{Quantized Models Viewed as POMDPs, the Quantized Q-Learning Algorithm and its Convergence to Near-Optimality}}\label{pomdp_sec}
In this section, we view the quantized MDPs as POMDPs with a quantizer channel.
\subsection{Q-Learning Algorithm for POMDPs}
For a partially observed MDP, the controller does not have full access to the state, but only a noisy version of the state is available, which are called the observations. Let $\mathds{Y}$ denote the observation space, which is assumed to be a finite set. The relation between the state variable $x$ and the observation variable $y$ is determined by a stochastic kernel (regular
conditional probability) $O$    from  $\mathds{X}$ to $\mathds{Y}$, such that
$O(\,\cdot\,|x)$ is a probability measure on the power set $P(\mathds{Y})$ of $\mathds{Y}$ for every $x
\in \mathds{X}$, and $O(A|\,\cdot\,): \mathds{X}\to [0,1]$ is a Borel measurable function for every $A \in P(\mathds{Y})$. In this setup, since the decision maker has access to a noisy version $Y_t$ of the state $X_t$ at each time step, the admissible policies are sequences of functions of observations and actions at each time step. 

The traditional Q-learning algorithm (\ref{Q_alg}) is constructed under the assumption that the controller can see the realizations of the state, and that the state is a controlled Markov chain. However, for POMDPs, the algorithm in (\ref{Q_alg}) is not directly applicable. A natural, though optimistic, suggestion to attempt to learn POMDPs would be to ignore the partial observability and pretend the noisy observations reflect the true state perfectly. Therefore, in this algorithm, the decision maker applies an arbitrary admissible policy $\gamma$  and collects realizations of observations, action, and stage-wise cost under this policy:
$$
Y_0,U_0,c(X_0,U_0),Y_1,U_1,c(X_1,U_1) \ldots. 
$$
Using this collection, it updates its Q-functions as follows: for $t\geq0$, if $(Y_t,U_t)=(x,u)$, then
\begin{align}\label{QPOMDP1}
&Q_{t+1}(Y_t,U_t)=(1-\alpha_t(Y_t,U_t)) \, Q_t(Y_t,U_t)\nonumber \\
&\phantom{xxxxxxxxxxxxxxxx}+\alpha_t(Y_t,U_t)\left(c(X_t,U_t)+\beta \min_{v \in \mathds{U}} Q_t(Y_{t+1},v)\right).
\end{align}
Note that the observation process is not a controlled Markov chain as past observations, when conditioned on the current observation,  can give  information about the future state variable, and so, can affect the next observation distribution. Second, the cost realization $c(X_t,U_t)$ depend on the observation $Y_t$ in a random and a time-dependent fashion. Indeed, for some $(y,u) \in \mathds{Y}\times\mathds{U}$, let $C_t(y,u)$ be a $\mathds{R}$-valued random variable with the following distribution:
$$
\Pr\left(C_t(y,u) \in \,\cdot\,\right) =P^{\gamma}_t(c(X_t,u) \in \cdot | Y_t=y,U_t=u), 
$$
where $P_t^{\gamma}$ is the conditional distribution of $X_t$ given $(Y_t,U_t)$ under the policy $\gamma$. Then, we can view $c(X_t,U_t)$ as a realization of the random variable $C_t(Y_t,U_t)$, where the expected value of the random variable $C_t(Y_t,U_t)$ is $\int_{\R} c(x,U_t) \, P_t^{\gamma}(dx|Y_t,U_t)$. 
For these two reasons, convergence of the above algorithm does not follow directly from usual techniques \citep{jaakkola1994convergence, TsitsiklisQLearning}. Additionally, even if the convergence is guaranteed, it is not immediate what the limit Q-function is, and whether it is meaningful at all. In particular, it is not known what MDP model gives rise to the limit Q-function. 

A general version of this approach is considered by \cite{kara2021convergence}, where the Q-iterations are constructed not only with the most recent observation but with a finite window of past observation and control action variables. It is shown that the algorithm is convergent under mild ergodicity conditions on the resulting state process $\{X_t\}_{t\geq0}$, and the limiting Q-function is an optimal Q-function of the approximate belief MDP. Furthermore, the learned policies are shown to be nearly optimal under some filter stability assumptions.

 The following assumptions will be imposed for the convergence.

\smallskip

\begin{assumption}\label{partial_q}
\hfill
\begin{itemize}
\item [(1.)] We let $\alpha_t(y,u)=0$ unless $(Y_t,U_t)=(y,u)$. Otherwise, let
\[\alpha_t(y,u) = {1 \over 1+ \sum_{k=0}^{t} 1_{\{Y_k=y, U_k=u\}} }.\]
\item [2.] Under the exploration policy $\gamma^*$,  $X_t$ is uniquely ergodic and thus has a unique invariant invariant measure $\pi_{\gamma^*}$. 
\item [3.] During the exploration phase, every observation-action pair $(y,u)$ is visited infinitely often.
\end{itemize}
\end{assumption}

{We note that a sufficient condition for the second item above is that the state process $\{X_t\}_{t\geq0}$ is positive Harris recurrent}. 



\smallskip

The following result is proved by \cite{kara2021convergence} (see also related results in  \cite{CsabaSmart} and \cite{singh1994learning}). It states the algorithm in (\ref{QPOMDP1}) converges to the optimal Q-function of a MDP whose system components can be described by transition kernel, observation channel, and stage-wise cost of the original POMDP.

\begin{theorem}\label{main_thm} {\cite[Theorem 4.1]{kara2021convergence}}
Under Assumption \ref{partial_q}, the algorithm given in (\ref{QPOMDP1}) converges almost surely to $Q^*$ which satisfies
\begin{align}\label{fixed}
Q^*(y,u)=C^*(y,u)+\beta\sum_{z \in \mathds{Y}} P^*(z|y,u)\min_{v \in \mathds{U}} Q^*(z,v).
\end{align}
Here, $P^*$ and $C^*$ are defined as follows 
\begin{align*}
P^*(y_1|y,u) &:=\int_{\mathds{X}}\int_{\mathds{X}}O(y_1|x_1) \, \mathcal{T}(dx_1|x_0,u) \, P^{\pi_{\gamma^*}}(dx_0|y)\\
C^*(y,u) &:=\int_{\mathds{X}}c(x_0,u)P^{\pi_{\gamma^*}}(dx_0|y)
\end{align*}
where $P^{\pi_{\gamma^*}}(dx_0|y)$ is the reverse channel induced by the observation channel $O$ and the invariant distribution $\pi_{\gamma^*}$ of the state process under exploration policy $\gamma^*$; that is,
$O(y|x) \, \pi_{\gamma^*}(dx) =  P^{\pi_{\gamma^*}}(dx|y) \, P^{\pi_{\gamma^*}}(y)$, where $P^{\pi_{\gamma^*}}(y) = \int_{\mathds{X}} O(y|x) \, \pi_{\gamma^*}(dx)$. In other words, 
$$
P^{\pi_{\gamma^*}}(A|y) := \frac{\int_{A}O(y|x) \, \pi_{\gamma^*}(dx)}{\int_{\mathds{X}} O(y|x) \, \pi_{\gamma^*}(dx)},
$$
almost surely. 
\end{theorem}

\subsection{{Quantized Models Viewed as POMDPs and the Quantized Q-Learning Algorithm}}
In this section, we consider the Q-learning algorithm in (\ref{q_alg2}) that is established for fully observed MDPs with a continuous state space. Before the convergence result, recall that we observed in Theorem \ref{weak:thm2} that any MDP with weakly continuous transition probability can be approximated by MDPs with finite action spaces. Thus, to make the presentation shorter, we will either assume that the action set is finite, or it has been approximated with arbitrarily small approximation error by a finite action set through the construction in Theorem \ref{weak:thm2}. Assuming finite action sets will help us avoid measurability issues \citep[see universal measurability discussions in][]{SaYuLi15c} as well as issues with existence of optimal policies.

As before, let $\mathds{Y}$ be a finite set, which will play a role for the approximations of $\mathds{X}$. Recall the quantization map  $q:\mathds{X}\to\mathds{Y}$  defined in (\ref{quant_map}) such that for any $x\in\mathds{X}$, $q(x)=y_i$ for some $y_i\in\mathds{Y}$, where $y_i$'s are the representative states for the collection of disjoint sets $\{B_i\}_{i=1}^M$ so that $\bigcup_i B_i=\mathds{X}$ and $B_i\bigcap B_j =\emptyset$ for any $i\neq j$. 

Let us recall the Q-learning algorithm in (\ref{q_alg2}). In this learning algorithm, the decision maker applies the exploration policy $\gamma^*$ and collects realizations of state, action, and stage-wise cost under this policy:
\begin{align*}
X_0,U_0,c(X_0,U_0),X_1,U_1,c(X_1,U_1) \ldots. 
\end{align*}
Using this collection, the Q-functions which are defined for the quantized state action pairs in $\mathds{Y} \times \mathds{U}$ are updated as follows: for $t\geq0$, if $(X_t,U_t)=(x,u)\in\mathds{X}\times\mathds{U}$, then
\begin{align}\label{q_alg3}
&Q_{t+1}(q(x),u)=(1-\alpha_t(q(x),u)) \, Q_t(q(x),u) \nonumber \\
&\phantom{xxxxxxxxxxxxxxxx}+\alpha_t(q(x),u)\left(c(x,u)+\beta \min_{v \in \mathds{U}} Q_t(q(X_{t+1}),v)\right).  
\end{align}
We interpret this iteration as a special case of the POMDP iteration (\ref{QPOMDP}) by considering the dicretization as a quantizer channel. If we consider the finite set $\mathds{Y}$ as the observation space and define the observation channel $O$ as $O(y_i|x)=1_{\{x\in B_i\}}$, for $i=1, \cdots, M$, then the algorithm in (\ref{q_alg3}) is the same algorithm as in $(\ref{QPOMDP})$. Therefore, the following result is then a direct corollary of Theorem \ref{main_thm}.

\begin{algorithm}[H] \label{q-algorithm}
\caption{{Learning Algorithm: Quantized Q-Learning}}
\label{Qit}

\begin{algorithmic}{
\STATE{Input: $Q_0$ (initial $Q$-function), $q:\mathds{X}\rightarrow\mathds{Y}$ (quantizer), $\gamma^*$ (exploration policy), $L$ (number of data points), $\{N(y,u)=0\}_{(y,u) \in \mathds{\sY\times\mathds{U}}}$ (number of visits to state-action pairs)}.
\STATE{Start with $Q_0$}
\FOR{$t=0,\ldots,L-1$}
\STATE{
\begin{itemize}
\item[\ding{43}] If $(X_t,U_t)$ is the current state-action pair $\Longrightarrow$ generate the cost $c(X_t,U_t)$ and the next state $X_{t+1} \sim \mathcal{T}(\,\cdot\,|X_t,U_t)$, and set 
$$N(q(X_t),U_t) = N(q(X_t),U_t) + 1.$$ 
\item[\ding{43}] Update $Q$-function $Q_t$ for the inputs $(q(X_t),U_t)$ as follows:\\
\begin{align}
&Q_{t+1}(q(X_t),U_t)=(1-\alpha_t(q(X_t),U_t)) \, Q_t(q(X_t),U_t) \nonumber \\
&\phantom{xxxxxxxxxxxxxxxx}+\alpha_t(q(X_t),U_t)\left(c(X_t,U_t)+\beta \min_{v \in \mathds{U}} Q_t(q(X_{t+1}),v)\right), \nonumber 
\end{align}
where 
$$
\alpha_t(q(X_t),U_t) = \frac{1}{1+N(q(X_t),U_t)}.
$$
\item[\ding{43}] Generate  $U_{t+1} \sim \gamma^*$.
\end{itemize}
}
\ENDFOR
\RETURN{$Q_L$}}
\end{algorithmic}
\end{algorithm}

\begin{theorem}\label{q_conv}
Under Assumption \ref{partial_q}, for every pair $(y_i,u)\in\mathds{Y}\times\mathds{U}$, the algorithm given above converges to
\begin{align*}
Q^*(y_i,u)=C^*(y_i,u)+\beta\sum_{y_j\in\mathds{Y}} P^*(y_j|y_i,u)\min_{v \in \mathds{U}}Q^*(y_j,v).
\end{align*}
Here, $P^*$ and $C^*$ are defined by
\begin{align}\label{q_limit}
C^*(y_i,u) &= \int_{B_i}c(x,u) \, \hat{\pi}_{y_i}^*(dx)\nonumber\\
P^*(y_j|y_i,u) &= \int_{B_i}\mathcal{T}(B_j|x,u)\, \hat{\pi}_{y_i}^*(dx),
\end{align}
where
\begin{align}
\hat{\pi}_{y_i}^*(A):=\frac{\pi_{\gamma^*}(A)}{\pi_{\gamma^*}(B_i)}, \quad \forall A\subset B_i, \quad \forall i\in \{1,\dots,M\},
\end{align}
and $\pi_{\gamma^*}$ is the invariant measure of the state process under the exploration policy $\gamma^*$.
\end{theorem}

\subsection{Error Analysis for {Convergence of Quantized Q-Learning} for Continuous Space MDPs}\label{error_analysis}

The model described in (\ref{finite_cost}) is the same model given by the equations (\ref{q_limit}). Hence, the result and error bounds from Section \ref{pomdp_sec} can be used for the error analysis of the Q-learning algorithm given by (\ref{q_alg3}). For the remainder of this section, 
we present a series of results for the performance of the policies learned through the approximate Q-learning algorithm in (\ref{q_alg3}) building on the results from Section \ref{pomdp_sec}. In these corollaries, it is always assumed that $\pi_{\gamma^*}(B_i)>0$ for all $i\in\{1,\dots,M\}$ where $B_i$'s are the quantization bins and $\pi_{\gamma^*}$ is the invariant measure on the state process under the exploration policy $\gamma^*$.

\subsubsection{Error Analysis for Non-Compact MDPs}
Our first result is in asymptotic nature and requires very mild conditions for the convergence (i.e., continuity of the stage-wise cost and weak continuity of the transition kernel). It follows from Theorem~\ref{q_conv} and Theorem~\ref{weak_thm}.

\begin{corollary}\label{cor3}
Under Assumption~\ref{partial_q} and Assumption~\ref{weak:as1}, the Q learning algorithm in (\ref{q_alg3}) converges to $Q^*$ in Theorem~\ref{q_conv} with probability 1 and for any policy $\hat{\gamma}$ that satisfies $Q^*(x,\hat{\gamma}(x))=\min_{u \in \mathds{U}} Q^*(x,u)$ (i.e., greedy policy of $Q^*$), for any compact $K \subset \mathds{X}$, we have
\begin{align*}
&\sup_{x_0 \in K} \left|J_\beta(x_0,\hat{\gamma})-J^*_\beta(x_0)\right|\to 0
\end{align*}
as $L^-\to 0$, where $L^-$ is defined in (\ref{unif_loss-2}).
\end{corollary}

We recall now that the error bounds to be presented in Corollary~\ref{cor1} and Corollary~\ref{cor2} below will involve the function $L$ and the uniform bound $\bar{L}$ which are defined as follows: for some $x\in\mathds{X}$ where $x$ belongs to a quantization bin $B_i$ whose representative state is $y_i$ (i.e. $q(x)=y_i$) and averaging measure $\hat{\pi}_{y_i}^*$, we have 
\begin{align*}
L(x)&:=\int_{B_i}\|x-x'\| \, \hat{\pi}_{y_i}^*(dx')\\
\bar{L}&:=\max_{i=1,\dots,M}\sup_{x,x'\in B_i}\|x-x'\|.
\end{align*}

The following result follows from Theorem~\ref{q_conv} and Theorem~\ref{tv_thm_robust}. 

\begin{corollary}\label{cor1}
Under Assumption~\ref{partial_q} and Assumption~\ref{tv_assmpt}, the Q-learning algorithm in (\ref{q_alg3}) converges to $Q^*$ in Theorem~\ref{q_conv} with probability 1 and for any policy $\hat{\gamma}$ that satisfies $Q^*(x,\hat{\gamma}(x))=\min_{u \in \mathds{U}} Q^*(x,u)$ (i.e., greedy policy of $Q^*$), for any initial state $x_0$, we have
\begin{align*}
\left|J_\beta(x_0,\hat{\gamma})-J^*_\beta(x_0)\right|\leq 2\left(\alpha_c+\frac{\beta\alpha_T\|c\|_\infty}{1-\beta}\right) \sum_{t=0}^\infty \beta^t \sup_{\gamma\in\Gamma} E^\gamma_{x_0}\left[L(X_t)\right].
\end{align*}
\end{corollary}

\subsubsection{Application to Models with Compact State Spaces}\label{compact_ref}

For the case with compact spaces, we obtain sharper bounds in the following.

The following result follows from Theorem~\ref{q_conv} and Theorem~\ref{wass_thm_robust}.

\begin{corollary}\label{cor2}
Under Assumption~\ref{partial_q} and Assumption~\ref{wass_assmpt}, the Q learning algorithm in (\ref{q_alg3}) converges to $Q^*$ in Theorem~\ref{q_conv} with probability 1 and for any policy $\hat{\gamma}$ that satisfies $Q^*(x,\hat{\gamma}(x))=\min_{u \in \mathds{U}} Q^*(x,u)$ (i.e., greedy policy of $Q^*$), we have
\begin{align*}
\sup_{x_0\in\mathds{X}}\left|J_\beta(x_0,\hat{\gamma})-J^*_\beta(x_0)\right|\leq  \frac{2\alpha_c}{(1-\beta)^2(1-\beta\alpha_T)}\bar{L}.
\end{align*}
where $\bar{L}$ is defined in (\ref{unif_loss}).
\end{corollary}

Building on the results presented, we now show that for compact state spaces, the terms $L(x)$ and the uniform bound $\bar{L}$ can be explicitly bounded via cardinality of finite approximating set $\mathds{Y}$ and dimension $d$ of the state space. To this end, we assume that the state space $\mathds{X}\subset \mathds{R}^d$ is compact, and thus totally bounded. Then, for a given $M$, we can quantize $\mathds{X}$ by choosing a finite subset $\mathds{Y}=\{y_1,\dots,y_M\}$ such that 
\begin{align*}
\max_{x\in\mathds{X}}\min_{y_i\in\mathds{Y}}\|x-y_i\|\leq \alpha (1/M)^{1/d}
\end{align*}
for some $\alpha>0$, which is possible since $\mathds{X}$ is totally bounded (\cite[Theorem 2.3.1]{Dud02}). Using this construction, one can then write the following immediate bounds:
\begin{align*}
L(x) &\leq 2\alpha (1/M)^{1/d}, \, \text{ for all } x \in \mathds{X},\\
\bar{L} &\leq 2\alpha (1/M)^{1/d}.
\end{align*}
We can then state the following results, which follow from Corollary~\ref{cor1} and Corollary~\ref{cor2}.

\begin{corollary}\label{cor5}
If the state space $\mathds{X}\subset \mathds{R}^d$ is compact, under Assumption~\ref{partial_q} and Assumption~\ref{tv_assmpt}, the Q-learning algorithm in (\ref{q_alg3}) converges to $Q^*$ in Theorem~\ref{q_conv} with probability 1 and for any policy $\hat{\gamma}$ that satisfies $Q^*(x,\hat{\gamma}(x))=\min_{u \in \mathds{U}} Q^*(x,u)$ (i.e., greedy policy of $Q^*$), for any initial state $x_0$, we have
\begin{align*}
\left|J_\beta(x_0,\hat{\gamma})-J^*_\beta(x_0)\right|\leq \left(\alpha_c+\frac{\beta\alpha_T\|c\|_\infty}{1-\beta}\right) \frac{4\alpha (1/M)^{1/d}}{1-\beta}
\end{align*}
\end{corollary}

\begin{corollary}
If the state space $\mathds{X}\subset \mathds{R}^d$ is compact, under Assumption~\ref{partial_q} and Assumption~\ref{wass_assmpt}, the Q learning algorithm in (\ref{q_alg3}) converges to $Q^*$ in Theorem~\ref{q_conv} with probability 1 and for any policy $\hat{\gamma}$ that satisfies $Q^*(x,\hat{\gamma}(x))=\min_{u \in \mathds{U}} Q^*(x,u)$ (i.e., greedy policy of $Q^*$), we have
\begin{align*}
\sup_{x_0\in\mathds{X}}\left|J_\beta(x_0,\hat{\gamma})-J^*_\beta(x_0)\right|\leq  \frac{4\alpha_c}{(1-\beta)^2(1-\beta\alpha_T)}{\alpha (1/M)^{1/d}}
\end{align*}
\end{corollary}


\begin{remark}\label{compremark}
A linear approximation for the Q-values can be obtained using a collection of linearly independent basis functions $\{\phi_1,\dots,\phi_M\}$ where $\phi_i:\mathds{X}\times\mathds{U}\to\mathds{R}$ for each $i$. One, then, is interested in the optimization of
\begin{align*}
Q_\theta(x,u)=\sum_{i=1}^M\phi_i(x,u)\theta(i)
\end{align*}
by optimizing over a parameter $\theta\in\mathds{R}^M$. In our results, for $i\in\{1,\dots,M\}$ the basis functions are of the following type
\begin{align*}
\phi_i(x,u)=\mathds{1}_{(B,u)_i}(x,u)
\end{align*}
where $(B,u)_i$ is a pair of a state space quantization bin under $q$ (see (\ref{quant_map})) and a control action such that $M=|\mathds{Y}|\times|\mathds{U}|$ where $\mathds{Y}$ is the finite subset of the original state space $\mathds{X}$.

We have shown that under the quantized Q learning algorithm (\ref{q_alg3}), for $(y,u)_i\in\mathds{Y}\times\mathds{U}$ the parameter $\theta(i)$ is obtained as 
\begin{align*}
\theta(i)= Q^*((y,u)_i)
\end{align*}
where $Q^*$ satisfies (\ref{fixed}). Furthermore, $Q^*$ is the optimal Q-value function for a finite MDP model. Hence, using finite MDP approximations (Section \ref{finite_app_sec}), we are able to provide further insight for the error term.

We should note that, even though the convergence result and the provided error bounds are valid under any quantization map $q$ (sometimes referred to as the state aggregation map in the literature), the choice of $q$ plays a crucial role on the performance of the approximations. A naive choice for $q$ is a uniform quantization map, however, different choices can also be made, e.g.
\begin{itemize}
\item One can group the less likely state and action pairs together, considering the error term can be bounded by the expectation of the loss function $L$, see Corollary \ref{cor1},
\item The quantization can be finer where the optimal Q values changes fast or the quantization can be made cruder at parts where the Q values change slower.
\end{itemize}
These different choices for the quantization structure affect the performance of the learning algorithm, however, the provided error bounds and the convergence result still hold for any map $q$, under the sufficient conditions provided in the paper.
\end{remark}

\section{Numerical Studies}\label{num_stud}
We present two numerical examples.
\subsection{A Fisheries Management Problem}

In this numerical example, we consider the following population growth model, called a Ricker model,  \citep[see][Section 7.2]{SaYuLi15c}:
\begin{align}
X_{t+1} = \theta_1 U_t \exp\{-\theta_2 U_t + V_t\}, \text{ } t=0,1,2,\ldots \label{aux9}
\end{align}
where $\theta_1, \theta_2 \in \R_{+}$, $X_t$ is the population size in season $t$, and $U_t$ is the population to be left for spawning for the next season, or in other words, $X_t - U_t$ is the amount of fish captured in the season $t$. The one-stage `reward' function is $r(X_t-U_t)$, where $r$ is some utility function. In this model, the goal is to maximize the discounted reward. Note that all results in this paper apply  with straightforward modifications for the case of maximizing reward instead of minimizing cost.

The state and action spaces are $\mathds{X}=\mathds{U}=[\kappa_{\min},\kappa_{\max}]$, for some $\kappa_{\min}, \kappa_{\max} \in \R_{+}$. Since the population left for spawning cannot be greater than the total population, for each $x \in \mathds{X}$, the set of admissible actions is $\mathds{A}(x)=[\kappa_{\min},x]$ which is not consistent with our assumptions. However, we can (equivalently) reformulate above problem so that the admissible actions $\mathds{A}(x)$ will become $\mathds{A}$ for all $x\in\mathds{X}$. In this case, instead of dynamics in equation (\ref{aux9}) we have
\begin{align}
X_{t+1} = \theta_1 \min(U_t,X_t) \exp\{-\theta_2 \min(U_t,X_t) + V_t\}, \text{ } t=0,1,2,\ldots \nonumber
\end{align}
and $\mathds{A}(x) = [\kappa_{\min},\kappa_{\max}]$ for all $x\in\mathds{X}$. The one-stage reward function is $r(X_t-U_t)1_{\{X_t\geq U_t\}}$.

The noise process $\{V_{t}\}$ is a sequence of independent and identically distributed (i.i.d.) random variables which are uniformly distributed on $[0,\lambda]$. For the numerical results, we use the following values of the parameters:
\begin{align}
\theta_1 = 1.1, \text{ } \theta_2=0.1, \text{ }\kappa_{\max}=7, \text{ }\kappa_{\min}=0, \text{ }\lambda=0.5, \text{ } \beta=0.5.\nonumber
\end{align}
The utility function $r$ is taken to be the shifted isoelastic utility function 
\begin{align}
r(z) = 3 \bigl((z+0.5)^{1/3}-(0.5)^{1/3}\bigr). \nonumber
\end{align}
We selected 20 different values for the number $M$ of grid points to discretize the state space: $M=10,20,30,\ldots,200$. The grid points are chosen uniformly over the interval $[\kappa_{\min},\kappa_{\max}]$. We also uniformly discretize the action space $\mathds{A}$ by using the number of $70$ grid points.

{
We first implement the value iteration algorithm to compute the optimal value functions of the finite models. Finite models are constructed as in Section~\ref{finitestate} using uniform distribution $\pi^*$ on $[\kappa_{\min},\kappa_{\max}]$. Note that  $\pi^*$ is not the invariant probability measure of the state processes induced by exploration policy $\gamma^*$, and thus, the learning algorithm may not exactly converge to the optimal value of the finite model when the number of grid points $M$ is small. However, the optimal value functions of finite models are proved to be converging to the optimal value function of the original model as $M$ becomes larger for any $\pi^*$. Hence, the learned value functions converge to the optimal value functions of the finite models obtained via $\pi^*$ as $M$ gets larger. After we run value iteration algorithm for finite models, we use the Quantized Q-learning algorithm in (\ref{q_alg3}) to obtain the approximate value functions of the discretized models using the data points coming from the original model. For each discretization, we gradually increase the training set proportional to the number of states in the discretized model to achieve a high accuracy when the number of grid points is large. Moreover, we also run the learning algorithm for five different episodes. Finally, we compare value functions obtained through value iteration and Q-learning.}


Figure~\ref{gr3} shows the graph of the optimal value functions of the finite models  and value functions given by Q-learning algorithm for five different runs corresponding to the different values of $M$ (number of grid points), when the initial state is $x=1.5$. It can be seen that the value functions are close to each other and converge to the optimal value function of the original model as $M$ increases.

\begin{figure}[h]
\centering
\includegraphics[width=6in, height=2.5in]{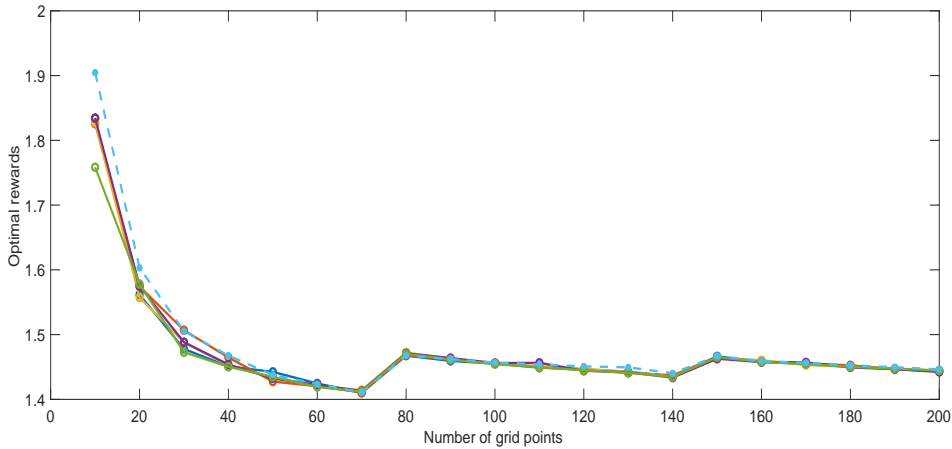}
\caption{Optimal rewards of the finite models (dashed curve) and learned rewards (other curves) when the initial state is $x=1.5$}
\label{gr3}
\end{figure}

{
\subsection{An Additive Noise System}

In this example, we consider the additive noise system:
\begin{align}
x_{t+1}=F(x_{t},a_{t})+v_{t}, \text{ } t=0,1,2,\ldots \nonumber
\end{align}
where $x_t, a_t, v_t \in \R$ and $\sX = \R$. Hence, the state space is non-compact. The noise process $\{v_{t}\}$ is a sequence of $\R$-valued i.i.d. random variables with common density $g$. The one-stage cost function is $c(x,a) = (x-a)^2$, the action space is $\sA = [-L,L]$ for some $L>0$. We assume that $g$ is a Gaussian probability density function with zero mean and variance $\sigma^2$.

For the numerical results, we use the following parameters: $F(x,a)=x+a$, $\beta=0.3$, $L=0.5$, and $\sigma = 0.1$.

We selected a sequence $\bigl\{[-l_n,l_n]\bigr\}_{n=1}^{24}$ of nested closed intervals, where $l_n = 0.5 + 0.25 n$, to approximate $\R$. Each interval is uniformly discretized. For the first half of the intervals $\bigl\{[-l_n,l_n]\bigr\}_{n=1}^{12}$, we use $0.1$ as the uniform bin length, and for the second half of the intervals $\bigl\{[-l_n,l_n]\bigr\}_{n=13}^{24}$, we use $0.05$ as the uniform bin length. Therefore, the discretization is refined after some point. For each $n$, the finite state space is given by $\{x_{n,i}\}_{i=1}^{k_n} \cup \{\Delta_n\}$, where $\{x_{n,i}\}_{i=1}^{k_n}$ are the representation points in the uniform quantization of the closed interval $[-l_n,l_n]$ and $\Delta_n$ is a pseudo state (see \cite[Section 3]{SaYuLi15c}). Here, the points outside of the interval $[-l_n,l_n]$ is mapped to the pseudo state by quantizer; that is, pseudo state $\Delta_n$ is the representation point of the overload region $\R \setminus [-l_n,l_n]$. We also uniformly discretize the action space $\sA = [-0.5,0.5]$ with $0.02$ as the length of the uniform bin.  For each $n$, the finite state models are constructed as in Section~\ref{finitestate} by using $\pi^*(\,\cdot\,) = \frac{1}{2} m_n(\,\cdot\,) + \frac{1}{2} \delta_{\Delta_n}(\,\cdot\,)$, where $m_n$ is the Lebesgue measure normalized over $[-l_n,l_n]$. We use the value iteration algorithm to compute the value functions of the finite models. Note that  $\pi^*$ is not the invariant probability measure of the state processes induced by exploration policy $\gamma^*$, and so, the learning algorithm may not exactly converge to the optimal value of the finite model when the number of grid points $M$ is small. However, it is known that optimal value functions of the finite models are proved to be converging to the optimal value function of the original model as $M$ becomes larger for any $\pi^*$. Hence, optimal value functions of the finite models obtained through invariant measure and $\pi^*$ converge to each other as $M$ gets larger. Hence, we expect that learned value functions converge to the optimal value functions of the finite models obtained via $\pi^*$.

After we run value iteration algorithm for finite models, we use the Q-learning algorithm in (\ref{q_alg3}) to obtain the approximate value functions of the discretized models using the data points coming from the original model. For each discretization, we again gradually increase the training set proportional to the number of states in the discretized model to achieve a high accuracy when the number of grid points are large. Moreover, we also run the learning algorithm for five different episodes. Finally, we compare value functions obtained through value iteration and Q-learning.

Figure~\ref{gr4} shows the graph of the optimal value functions of the finite models  and value functions given by Q-learning algorithm for five different runs corresponding to the different values of $M$ (number of grid points), when the initial state is $x=0.7$. It can be seen that the value functions are close to each other and converge to the optimal value function of the original model as $M$ increases.

\begin{figure}[h]
\centering
\includegraphics[width=6in, height=2.5in]{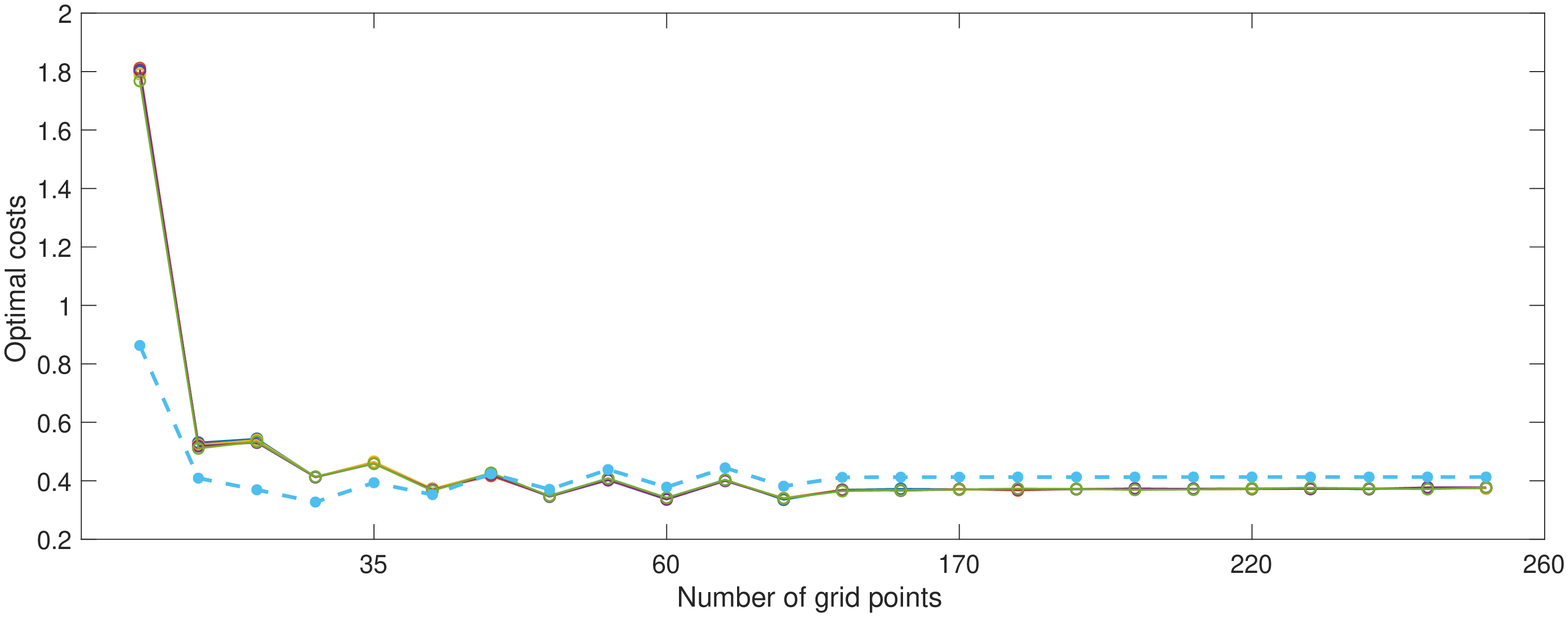}
\caption{Optimal costs of the finite models (dashed curve) and learned costs (other curves) when the initial state is $x=0.7$}
\label{gr4}
\end{figure}

}

\appendix

\section{Proof of Theorem \ref{tv_thm}}\label{tv_thm_proof}
\begin{proof}
We start by writing the DCOE for the the cost functions $J^*_\beta$
\begin{align*}
J^*_\beta(x_0)=\inf_{u\in\mathds{U}}\left\{c(x_0,u)+\beta \int J^*_\beta(x_1)\mathcal{T}(dx_1|x_0,u)\right\}.
\end{align*}

The DCOE for the the cost functions $\hat{J}_\beta$ can be written as
\begin{align*}
\hat{J}_\beta(x_0)=\inf_{u\in\mathds{U}}\left\{\int_{x\in B_0} c(x,u)\hat{\pi}_{x_0}(dx) +\beta \int_{x\in B_0}\int_{x_1\in\mathds{X}}\hat{J}_\beta(x_1) \mathcal{T}(dx_1|x,u)\hat{\pi}_{x_0}(dx)\right\}
\end{align*}
where $\hat{\pi}_{x_0}(dx)$ is the normalized  measure defined on the set $B_0$ such that the $B_0$ is the quantization bin $x_0$ belongs to. We note that, to write the DCOE in this alternative form, we used the fact that $\hat{J}_\beta(x)$ is constant over the quantization bins.  

Having these, now we can write
\begin{align*}
&\left|\hat{J}_\beta(x_0)-J^*_\beta(x_0)\right|\leq \sup_{u\in\mathds{U}}\left|c(x_0,u)-\int_{x\in B_0} c(x,u)\hat{\pi}_{x_0}(dx)\right| \\
&\qquad+ \beta \sup_{u\in\mathds{U}}\left| \int J^*_\beta(x_1)\mathcal{T}(dx_1|x_0,u)-\int_{x\in B_0}\int_{x_1\in\mathds{X}}\hat{J}_\beta(x_1) \mathcal{T}(dx_1|x,u)\hat{\pi}_{x_0}(dx)\right|\\
&= \sup_{u\in\mathds{U}}\left|\int_{x\in B_0}\left(c(x_0,u)- c(x,u)\right)\hat{\pi}_{x_0}(dx)\right|\\
&\qquad+ \beta\sup_{u\in\mathds{U}}\left|\int_{x\in B_0}\left( \int_{x_1\in\mathds{X}} J^*_\beta(x_1)\mathcal{T}(dx_1|x_0,u)-\int_{x_1\in\mathds{X}}\hat{J}_\beta(x_1) \mathcal{T}(dx_1|x,u)\right)\hat{\pi}_{x_0}(dx)\right|
\end{align*}
where for the last step we used the fact $c(x_0,u)$ and $\int_{x_1\in\mathds{X}}\hat{J}_\beta(x_1) \mathcal{T}(dx_1|x_0,u)$ are constants for the integration over the set $B_0$ under $\hat{\pi}_{x_0}(dx)$.

For the first term above, we write
\begin{align*}
& \sup_{u\in\mathds{U}}\left|\int_{x\in B_0}\left(c(x_0,u)- c(x,u)\right)\hat{\pi}_{x_0}(dx)\right|\leq  \sup_{u\in\mathds{U}}\int_{x\in B_0}\alpha_c|x_0-x|\hat{\pi}_{x_0}(dx)=\alpha_cL(x_0).
\end{align*}

For the second term, we start by adding and subtracting $\int_{x\in B_0}\int_{x_1\in\mathds{X}}\hat{J}_\beta(x_1) \mathcal{T}(dx_1|x_0,u)\hat{\pi}_{x_0}(dx)$ and we write
\begin{align*}
& \sup_{u\in\mathds{U}}\left|\int_{x\in B_0}\left( \int J^*_\beta(x_1)\mathcal{T}(dx_1|x_0,u)-\int_{x_1\in\mathds{X}}\hat{J}_\beta(x_1) \mathcal{T}(dx_1|x,u)\right)\hat{\pi}_{x_0}(dx)\right|\\
&\leq \int_{x\in B_0}\int_{x_1\in\mathds{X}}\left|J^*_\beta(x_1)-\hat{J}_\beta(x_1)\right| \mathcal{T}(dx_1|x_0,u)\hat{\pi}_{x_0}(dx)\\
&\qquad + \int_{x\in B_0}\left|\int_{x_1\in\mathds{X}}\hat{J}_\beta(x_1) \mathcal{T}(dx_1|x_0,u)-\int_{x_1\in\mathds{X}}\hat{J}_\beta(x_1) \mathcal{T}(dx_1|x,u)\right|\hat{\pi}_{x_0}(dx)\\
&\leq \sup_{\gamma\in\Gamma}E_{x_0}^\gamma\left[\left|J^*_\beta(X_1)-\hat{J}_\beta(X_1)\right|\right]+\int_{x\in B_0}\|\hat{J}_\beta\|_\infty\alpha_T|x_0-x|\hat{\pi}_{x_0}(dx)\\
&= \sup_{\gamma\in\Gamma}E_{x_0}^\gamma\left[\left|J^*_\beta(X_1)-\hat{J}_\beta(X_1)\right|\right]+\|\hat{J}_\beta\|_\infty\alpha_TL(x_0).
\end{align*}

Combining what we have so far
\begin{align*}
&\left|\hat{J}_\beta(x_0)-J^*_\beta(x_0)\right|\leq \alpha_cL(x_0)+\|\hat{J}_\beta\|_\infty\alpha_T\beta L(x_0)+\beta\sup_{\gamma\in\Gamma}E_{x_0}^\gamma\left[\left|J^*_\beta(X_1)-\hat{J}_\beta(X_1)\right|\right].
\end{align*}
Repeating the same steps for $E_{x_0}^\gamma\left[\left|J^*_\beta(X_1)-\hat{J}_\beta(X_1)\right|\right]$, we can have
\begin{align*}
&\left|\hat{J}_\beta(x_0)-J^*_\beta(x_0)\right|\leq \left(\alpha_c+\beta\alpha_T\|\hat{J}_\beta\|_\infty\right)\sum_{t=0}^1\beta^t\sup_{\gamma\in\Gamma}E^\gamma_{x_0}\left[L(X_t)\right]+\beta^2\sup_{\gamma\in\Gamma}E_{x_0}^\gamma\left[\left|J^*_\beta(X_2)-\hat{J}_\beta(X_2)\right|\right].
\end{align*}
By repeating this procedure, since $c$ is bounded, we can conclude that
\begin{align*}
&\left|\hat{J}_\beta(x_0)-J^*_\beta(x_0)\right|\leq \left(\alpha_c+\beta\alpha_T\|\hat{J}_\beta\|_\infty\right)\sum_{t=0}^\infty\beta^t\sup_{\gamma\in\Gamma}E^\gamma_{x_0}\left[L(X_t)\right].
\end{align*}
The proof follows by noting that $\|\hat{J}_\beta\|_\infty\leq \frac{\|c\|_\infty}{1-\beta}$.
\end{proof}

\section{Proof of Theorem \ref{tv_thm_robust}}\label{tv_thm_robust_proof}
\begin{proof}
With $\hat{\gamma}$ being optimal for the approximate model, by the triangle inequality we have
\begin{align*}
\left|J_\beta(x_0,\hat{\gamma})-J^*_\beta(x_0)\right|\leq \left|J_\beta(x_0,\hat{\gamma})-\hat{J}_\beta(x_0)\right|+\left|\hat{J}_\beta(x_0)-J^*_\beta(x_0)\right|.
\end{align*}
Note that the second term is bounded by Theorem \ref{tv_thm}. We now focus on the first term. We write the following value function iterations for $J_\beta(x_0,\hat{\gamma})$:
\begin{align*}
v_{k+1}(x_0)=c(x_0,\hat{\gamma}(x_0))+\beta\int v_k(x_1)\mathcal{T}(dx_1|x_0,\hat{\gamma}(x_0))
\end{align*}
for $v_0(x_0)=c(x_0,\hat{\gamma}(x_0))$.

Furthermore, the value function iterations for $\hat{J}_\beta(x_0)$ can be written as
\begin{align*}
\hat{v}_{k+1}(x_0)=\int_{x\in B_0} c(x,\hat{\gamma}(x_0))\hat{\pi}_{x_0}(dx) +\beta \int_{x\in B_0}\int_{x_1\in\mathds{X}}\hat{v}_k(x_1) \mathcal{T}(dx_1|x,\hat{\gamma}(x_0))\hat{\pi}_{x_0}(dx)
\end{align*}
where $\hat{v}_0(x_0)=\int_{x\in B_0} c(x,\hat{\gamma}(x_0))\hat{\pi}_{x_0}(dx)$ such that $B_0$ is the set $x_0$ belongs to.

For the value functions approximations, we have the following uniform bounds using the fact that the dynamic programming operator is a contraction under the supremum norm with modulus $\beta$:
\begin{align}\label{value_bounds}
\big|\hat{J}_\beta(x_0)-\hat{v}_k(x)\big|\leq \|c\|_\infty\frac{\beta^k}{1-\beta},\quad \big|J_\beta(x_0,\hat{\gamma})-v_k(x_0)\big|\leq \|c\|_\infty\frac{\beta^k}{1-\beta}.
\end{align}

We now claim and prove by induction that 
\begin{align*}
\left|\hat{v}_{k}(x_0)-v_{k}(x_0)\right|\leq \left(\alpha_c+\frac{\beta\|c\|_\infty\alpha_T}{1-\beta}\right)\sum_{t=0}^{k-1}\beta^t\sup_{\gamma\in\Gamma}E_{x_0}^\gamma\left[L(X_t)\right]+\beta^k\alpha_c\sup_{\gamma\in\Gamma}E_{x_0}^\gamma\left[L(X_k)\right].
\end{align*}
For $k=0$:
\begin{align*}
v_{0}(x_0)=c(x_0,\hat{\gamma}(x_0)),\qquad
\hat{v}_{0}(x_0)=\int_{x\in B_0} c(x,\hat{\gamma}(x_0))\hat{\pi}_{x_0}(dx) 
\end{align*}
it can be seen that  $|v_0(x_0)-\hat{v}_0(x_0)|\leq \alpha_c\int_{B_0}\|x_0-x\|\hat{\pi}_{x_0}^*(dx)=\alpha_cL(x_0)$.

For a general $k$, we write
\begin{align*}
&\left|v_{k+1}(x_0)-\hat{v}_{k+1}(x_0)\right|\leq \int_{x\in B_0}\left|c(x_0,\hat{\gamma}(x_0))- c(x,\hat{\gamma}(x_0))\right|\hat{\pi}_{x_0}(dx)\\
& +\beta \int_{x\in B_0}\left|\int_{x_1\in\mathds{X}}v_k(x_1)\mathcal{T}(dx_1|x_0,\hat{\gamma}(x_0))-\int_{x_1\in\mathds{X}}\hat{v}_k(x_1) \mathcal{T}(dx_1|x,\hat{\gamma}(x_0))\right|\hat{\pi}_{x_0}(dx)\\
&\leq \alpha_c L(x_0)+\beta \left|\int_{x_1\in\mathds{X}}v_k(x_1)\mathcal{T}(dx_1|x_0,\hat{\gamma}(x_0))-\int_{x_1\in\mathds{X}}\hat{v}_k(x_1) \mathcal{T}(dx_1|x_0,\hat{\gamma}(x_0))\right|\\
&\qquad+\beta \int_{x\in B_0}\left|\int_{x_1\in\mathds{X}}\hat{v}_k(x_1) \mathcal{T}(dx_1|x_0,\hat{\gamma}(x_0))-\int_{x_1\in\mathds{X}}\hat{v}_k(x_1) \mathcal{T}(dx_1|x,\hat{\gamma}(x_0))\right|\hat{\pi}_{x_0}(dx)\\
&\leq \alpha_c L(x_0)+\beta\sup_{\gamma\in\Gamma}E^\gamma_{x_0}\left[|v_k(X_1)-\hat{v}_k(X_1)|\right]+\beta \|\hat{v}_k\|_\infty\alpha_T\int_{x\in B_0}\|x-x_0\|\hat{\pi}_{x_0}(dx)\\
&\leq \alpha_c L(x_0)+\beta\sup_{\gamma\in\Gamma}E^\gamma_{x_0}\left[ \left(\alpha_c+\frac{\beta\|c\|_\infty\alpha_T}{1-\beta}\right)\sum_{t=1}^{k}\beta^{t-1}\sup_{\gamma\in\Gamma}E_{X_1}^\gamma\left[L(X_t)\right]+\beta^k\alpha_c\sup_{\gamma\in\Gamma}E_{X_1}^\gamma\left[L(X_{k+1})\right]\right]\\
&\quad+\beta \|\hat{J}_\beta\|_\infty\alpha_TL(x_0)\\
&\leq \left(\alpha_c+\frac{\beta\|c\|_\infty\alpha_T}{1-\beta}\right)\sum_{t=0}^{k}\beta^t\sup_{\gamma\in\Gamma}E_{x_0}^\gamma\left[L(X_t)\right]+\beta^{k+1}\alpha_c\sup_{\gamma\in\Gamma}E_{x_0}^\gamma\left[L(X_{k+1})\right]
\end{align*}
where for the last two inequalities, we used the induction step and law of iterated expectations with the fact that $\|\hat{v}_k\|_\infty\leq \|\hat{J}_\beta\|_\infty\leq\frac{\|c\|_\infty}{1-\beta}$.

Using (\ref{value_bounds}), we can conclude that
\begin{align*}
 \left|J_\beta(x_0,\hat{\gamma})-\hat{J}_\beta(x_0)\right|\leq \left(\alpha_c+\frac{\beta\|c\|_\infty\alpha_T}{1-\beta}\right)\sum_{t=0}^{\infty}\beta^t\sup_{\gamma\in\Gamma}E_{x_0}^\gamma\left[L(X_t)\right]
\end{align*}
which completes the proof together with Theorem \ref{tv_thm}.
\end{proof}

\section{Proof of Theorem \ref{wass_thm}}\label{wass_thm_proof}
\begin{proof}
As in the proof of Theorem \ref{tv_thm}, we start with
\begin{align*}
&\left|\hat{J}_\beta(x_0)-J^*_\beta(x_0)\right|\leq \sup_{u\in\mathds{U}}\left|c(x_0,u)-\int_{x\in B_0} c(x,u)\hat{\pi}_{x_0}(dx)\right| \\
&\qquad+ \beta \sup_{u\in\mathds{U}}\left| \int J^*_\beta(x_1)\mathcal{T}(dx_1|x_0,u)-\int_{x\in B_0}\int_{x_1\in\mathds{X}}\hat{J}_\beta(x_1) \mathcal{T}(dx_1|x,u)\hat{\pi}_{x_0}(dx)\right|\\
&= \sup_{u\in\mathds{U}}\left|\int_{x\in B_0}c(x_0,u)- c(x,u)\hat{\pi}_{x_0}(dx)\right|\\
&\qquad+ \beta\sup_{u\in\mathds{U}}\left|\int_{x\in B_0}\left( \int J^*_\beta(x_1)\mathcal{T}(dx_1|x_0,u)-\int_{x_1\in\mathds{X}}\hat{J}_\beta(x_1) \mathcal{T}(dx_1|x,u)\right)\hat{\pi}_{x_0}(dx)\right|
\end{align*}
For the first term, we write
\begin{align*}
\sup_{u\in\mathds{U}}\left|\int_{x\in B_0}c(x_0,u)- c(x,u)\hat{\pi}_{x_0}(dx)\right|\leq  \sup_{u\in\mathds{U}}\int_{x\in B_0}\alpha_c|x_0-x|\hat{\pi}_{x_0}(dx)\leq \alpha_c\bar{L}.
\end{align*}
For the second term:
\begin{align*}
& \sup_{u\in\mathds{U}}\left|\int_{x\in B_0}\left( \int J^*_\beta(x_1)\mathcal{T}(dx_1|x_0,u)-\int_{x_1\in\mathds{X}}\hat{J}_\beta(x_1) \mathcal{T}(dx_1|x,u)\right)\hat{\pi}_{x_0}(dx)\right|\\
&\leq \int_{x\in B_0}\int_{x_1\in\mathds{X}}\left|J^*_\beta(x_1)-\hat{J}_\beta(x_1)\right| \mathcal{T}(dx_1|x,u)\hat{\pi}_{x_0}(dx)\\
&\qquad + \int_{x\in B_0}\left|\int_{x_1\in\mathds{X}}J^*_\beta(x_1) \mathcal{T}(dx_1|x_0,u)-\int_{x_1\in\mathds{X}}J^*_\beta(x_1) \mathcal{T}(dx_1|x,u)\right|\hat{\pi}_{x_0}(dx)\\
&\leq \sup_{x\in\mathds{X}}\left|\hat{J}_\beta(x)-J^*_\beta(x)\right|+\alpha_T\bar{L} \|J^*_\beta\|_L
\end{align*}
where $\|J^*_\beta\|_L$ denotes the Lipschitz constant of $J^*_\beta$ that is $\|J^*_\beta\|_L:=\sup_{x\neq x'}\frac{|J^*_\beta(x)-J^*_\beta(x')|}{|x-x'|}$. 

Combining what we have, we write
\begin{align*}
\sup_{x\in\mathds{X}}\left|\hat{J}_\beta(x)-J^*_\beta(x)\right|\leq \alpha_c \bar{L} + \beta\sup_{x\in\mathds{X}}\left|\hat{J}_\beta(x)-J^*_\beta(x)\right|+\beta\alpha_T\bar{L} \|J^*_\beta\|_L.
\end{align*}
Hence, we can conclude
\begin{align*}
\sup_{x\in\mathds{X}}\left|\hat{J}_\beta(x)-J^*_\beta(x)\right|\leq\frac{\alpha_c+\beta\alpha_T\|J_\beta^*\|_L}{1-\beta}\bar{L}.
\end{align*}
The result follows by noting that $\|J_\beta^*\|_L\leq \frac{\alpha_c}{1-\beta\alpha_T}$ \cite[Theorem 4.37]{SaLiYuSpringer}.
\end{proof}

\section{Proof of Theorem \ref{wass_thm_robust}}\label{wass_thm_robust_proof}
\begin{proof}
We again begin with the following initial bound:
\begin{align*}
\left|J_\beta(x_0,\hat{\gamma})-J^*_\beta(x_0)\right|\leq \left|J_\beta(x_0,\hat{\gamma})-\hat{J}_\beta(x_0)\right|+\left|\hat{J}_\beta(x_0)-J^*_\beta(x_0)\right|.
\end{align*}
Note that the second term is bounded by Theorem \ref{wass_thm}. We now focus on the first term. We have the following fixed point equation for $J_\beta(x_0,\hat{\gamma})$:
\begin{align*}
J_\beta(x_0,\hat{\gamma})=c(x_0,\hat{\gamma}(x_0))+\beta\int J_\beta(x_1,\hat{\gamma})\mathcal{T}(dx_1|x_0,\hat{\gamma}(x_0))
\end{align*}

Furthermore, the following fixed point equation can be written for $\hat{J}_\beta(x_0)$
\begin{align*}
\hat{J}_\beta(x_0)=\int_{x\in B_0} c(x,\hat{\gamma}(x_0))\hat{\pi}_{x_0}(dx) +\beta \int_{x\in B_0}\int_{x_1\in\mathds{X}}\hat{J}_\beta(x_1) \mathcal{T}(dx_1|x,\hat{\gamma}(x_0))\hat{\pi}_{x_0}(dx).
\end{align*}
With the given fixed point equations, we can write
\begin{align*}
&\left|J_\beta(x_0,\hat{\gamma})-\hat{J}_\beta(x_0)\right|\leq \alpha_c \bar{L}+\beta\int \left|J_\beta(x_1,\hat{\gamma})-\hat{J}_\beta(x_1)\right|\mathcal{T}(dx_1|x_0,\hat{\gamma}(x_0))\\
&\qquad+\beta \int_{x\in B_0}\left|\int_{x_1\in\mathds{X}}\hat{J}_\beta(x_1) \mathcal{T}(dx_1|x_0,\hat{\gamma}(x_0))-\int_{x_1\in\mathds{X}}\hat{J}_\beta(x_1) \mathcal{T}(dx_1|x,\hat{\gamma}(x_0))\right|\hat{\pi}_{x_0}(dx)\\
&\leq\alpha_c \bar{L}+\beta\sup_{x_0\in\mathds{X}}\left|J_\beta(x_0,\hat{\gamma})-\hat{J}_\beta(x_0)\right|\\
&\qquad+\beta \int_{x\in B_0}\left|\int_{x_1\in\mathds{X}}\hat{J}_\beta(x_1) \mathcal{T}(dx_1|x_0,\hat{\gamma}(x_0))-\int_{x_1\in\mathds{X}}J^*_\beta(x_1) \mathcal{T}(dx_1|x_0,\hat{\gamma}(x_0))\right|\hat{\pi}_{x_0}(dx)\\
&\qquad+\beta \int_{x\in B_0}\left|\int_{x_1\in\mathds{X}}J^*_\beta(x_1) \mathcal{T}(dx_1|x_0,\hat{\gamma}(x_0))-\int_{x_1\in\mathds{X}}J^*_\beta(x_1) \mathcal{T}(dx_1|x,\hat{\gamma}(x_0))\right|\hat{\pi}_{x_0}(dx)\\
&\qquad+\beta \int_{x\in B_0}\left|\int_{x_1\in\mathds{X}}J^*_\beta(x_1) \mathcal{T}(dx_1|x,\hat{\gamma}(x_0))-\int_{x_1\in\mathds{X}}\hat{J}_\beta(x_1) \mathcal{T}(dx_1|x,\hat{\gamma}(x_0))\right|\hat{\pi}_{x_0}(dx)\\
&\leq\alpha_c\bar{L}+\beta \sup_{x_0\in\mathds{X}}\left|J_\beta(x_0,\hat{\gamma})-\hat{J}_\beta(x_0)\right|+2\beta\sup_{x_0\in\mathds{X}}\left|\hat{J}_\beta(x_0)-J^*_\beta(x_0)\right|+\beta\|J_\beta^*\|_L\alpha_T\bar{L}
\end{align*}
Using Theorem \ref{wass_thm} and noting that  $\|J_\beta^*\|_L\leq \frac{\alpha_c}{1-\beta\alpha_T}$  (\cite[Theorem 4.37]{SaLiYuSpringer}), we can write
\begin{align*}
&\left|J_\beta(x_0,\hat{\gamma})-\hat{J}_\beta(x_0)\right|\leq \left(\alpha_c+\frac{2\beta\alpha_c}{(1-\beta\alpha_T)(1-\beta)}+\frac{\beta\alpha_c\alpha_T}{1-\beta\alpha_T}\right)\bar{L}+\beta \sup_{x_0\in\mathds{X}}\left|J_\beta(x_0,\hat{\gamma})-\hat{J}_\beta(x_0)\right|
\end{align*}
Thus, by taking the supremum on the left hand side, we can conclude
\begin{align*}
 \sup_{x_0\in\mathds{X}}\left|J_\beta(x_0,\hat{\gamma})-\hat{J}_\beta(x_0)\right|\leq  \frac{(1+\beta)\alpha_c}{(1-\beta)^2(1-\beta\alpha_T)}{\bar{L}}.
\end{align*}
Finally, by collecting everything we have so far, we can write
\begin{align*}
\left|J_\beta(x_0,\hat{\gamma})-J^*_\beta(x_0)\right|&\leq  \frac{(1+\beta)\alpha_c}{(1-\beta)^2(1-\beta\alpha_T)}{\bar{L}}+\frac{\alpha_c}{(1-\beta\alpha_T)(1-\beta)}\bar{L}\\
&\leq \frac{2\alpha_c}{(1-\beta)^2(1-\beta\alpha_T)}\bar{L}.
\end{align*}
\end{proof}


\end{document}